\documentclass{article}

\usepackage{aimzo_preprint,times}

\usepackage{amsmath,amsfonts,bm}

\def\eqref#1{equation~\ref{#1}}

\def\1{\bm{1}}

\DeclareMathAlphabet{\mathsfit}{\encodingdefault}{\sfdefault}{m}{sl}
\SetMathAlphabet{\mathsfit}{bold}{\encodingdefault}{\sfdefault}{bx}{n}

\newcommand{\E}{\mathbb{E}}

\newcommand{\R}{\mathbb{R}}

\usepackage{amsmath,amssymb,mathtools}
\usepackage{amsthm}
\usepackage{booktabs}
\usepackage{array}
\usepackage{graphicx}
\usepackage{algorithm}
\usepackage{algpseudocode}
\usepackage{microtype}
\usepackage{placeins}
\usepackage{enumitem}
\usepackage{xcolor}
\usepackage{xcolor}
\definecolor{linkblue}{RGB}{70,130,180}

\usepackage[
    colorlinks=true,
    linkcolor=linkblue,
    citecolor=linkblue,
    urlcolor=linkblue
]{hyperref}
\usepackage{url}

\newcommand{\method}{\textsc{AIM-ZO}}

\newtheorem{proposition}{Proposition}
\newtheorem{lemma}{Lemma}
\newtheorem{theorem}{Theorem}
\newtheorem{assumption}{Assumption}

\title{AIM-ZO: Activation-Informed Subspace Maintenance for Zeroth-Order LLM Fine-Tuning}

\author{
Yue Xie\textsuperscript{1}\thanks{Equal contribution},
Zhi Zheng\textsuperscript{2}\textsuperscript{*},
Yunpeng Ba\textsuperscript{1},
Xuyang Wu\textsuperscript{1},
Xialiang Tong\textsuperscript{3},
Zhichao Lu\textsuperscript{4},\\
\ \textbf{Tao Zhong\textsuperscript{3},
Zhenkun Wang\textsuperscript{1}}\\
\textsuperscript{1}Southern University of Science and Technology \quad
\textsuperscript{2}National University of Singapore \\
\textsuperscript{3}Huawei Technologies Ltd. \quad
\textsuperscript{4}City University of Hong Kong \\
\texttt{wuxy6@sustech.edu.cn, zhi.zheng@u.nus.edu} \\ 
}

\iclrfinalcopy

\begin{document}

\maketitle

\begin{abstract}

Zeroth-order (ZO) optimization offers a memory-efficient alternative for LLM fine-tuning by estimating updates only from forward evaluations of perturbed parameters, without backpropagation or activation storage. However, in billion-parameter LLMs, isotropic perturbations often waste many forward evaluations on weakly informative directions. To make these evaluations more informative, existing ZO methods restrict perturbations to low-dimensional subspaces. Yet the quality of these subspaces is critical: overly compressed or poorly maintained spaces can miss useful update directions. \textbf{To obtain a high-quality subspace for ZO updates}, this paper proposes AIM-ZO, a ZO fine-tuning method based on Activation-Informed Subspace Maintenance. AIM-ZO uses forward activations as local directional information and continuously integrates them into a broad, evolving subspace over training. To access broader gradient-relevant structure while keeping individual perturbations low-dimensional, AIM-ZO activates only a smaller set of shared and sampled directions, decoupling the maintained width from the active width. We evaluate AIM-ZO across 5 LLMs and 11 downstream tasks under matched forward-evaluation budgets; its six-task average exceeds the strongest fully evaluated ZO baseline by 1.26 percentage points on OPT-2.7B and MeZO by 2.85 percentage points on OPT-30B. Our code is available at \url{https://github.com/EkkoXy/AIM-ZO}
\end{abstract}

\vspace{-8pt}
\section{Introduction}
\label{sec:introduction}

Fine-tuning large language models (LLMs) requires substantial memory because backpropagation stores intermediate activations for gradient computation. Zeroth-order (ZO) optimization avoids this cost by estimating gradients from forward loss evaluations at perturbed parameter points \citep{ghadimi2013stochastic,nesterov2017random,zheng2026agenticesopt,ba2026evolutionstrategies}. \textbf{MeZO} shows that this approach can fine-tune LLMs with a memory footprint close to inference \citep{malladi2023mezo}. However, MeZO samples perturbation directions randomly from the full parameter space. As illustrated on the left of Figure~\ref{fig:overview}, in billion-parameter models, many such random directions \textbf{are poorly aligned with useful update directions}, so a large fraction of forward evaluations provide only \textbf{weak optimization signals}.

\begin{figure}[t]
  \centering\vspace{-1.5em}
  \includegraphics[width=\textwidth]{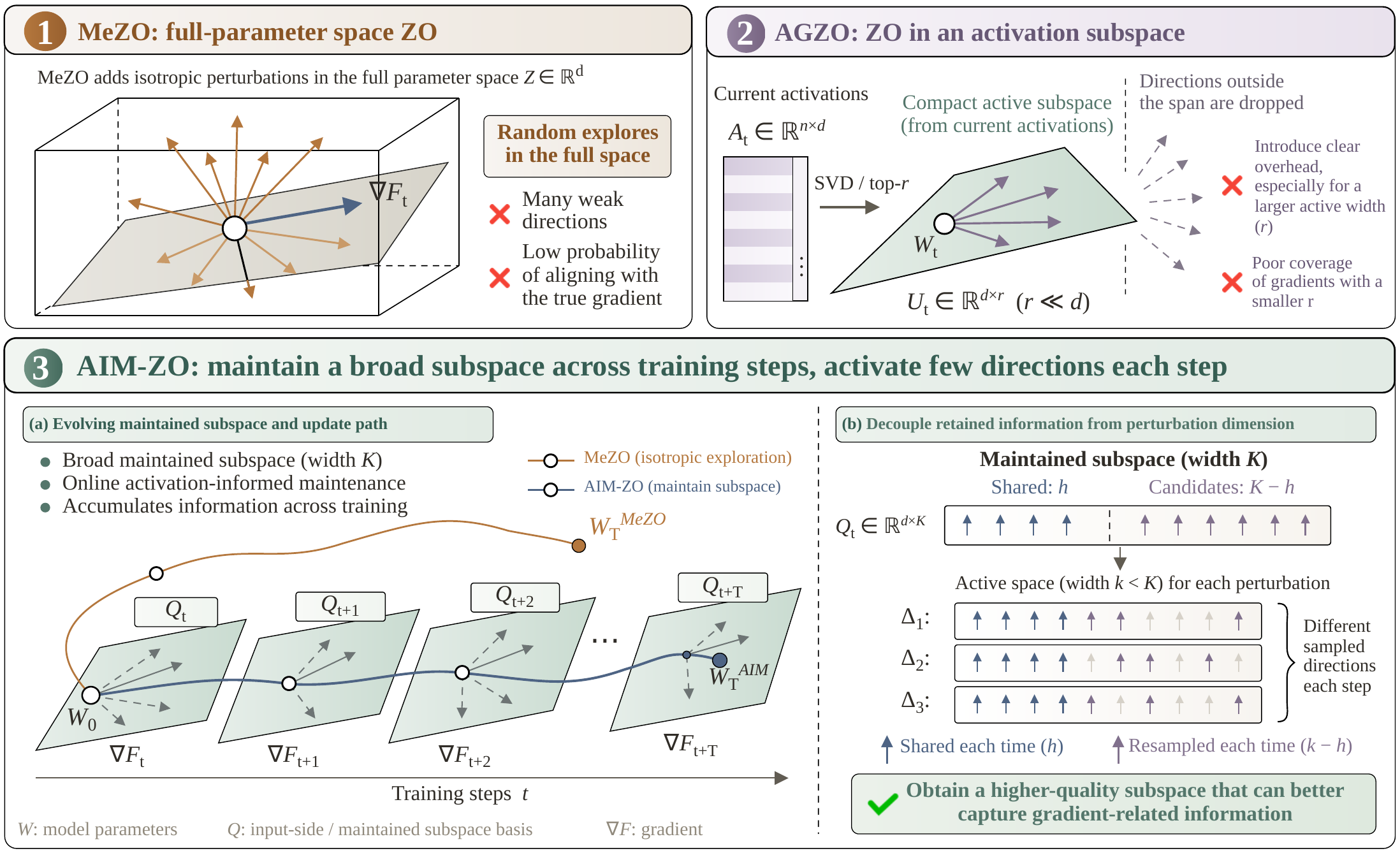}\vspace{-8pt}
  \caption{Perturbation spaces. MeZO explores the full parameter space; AGZO constructs a subspace from current activations. AIM-ZO maintains a width-$K$ subspace across training and activates a width-$k$ subspace with $h$ shared and $k-h$ sampled directions per perturbation.}
  \label{fig:overview}\vspace{-1em}
\end{figure}

Prior studies have shown that gradients can concentrate in low-dimensional subspaces that remain relatively stable over parts of training \citep{gurari2018tiny,jaiswal2025stabilization}. Building on this observation, recent ZO methods exploit low-dimensional structure through low-rank perturbations, random projections, or forward activations \citep{chen2025lozo,yu2025subzero,lin2026agzo,dong2026zoact}. Their effectiveness, however, depends critically on the quality of the perturbation subspace. Specifically, \textbf{1)} random or overly compressed perturbation subspaces may \textbf{miss useful gradient information}; and \textbf{2)} simply enlarging the perturbation subspace can capture more gradient information but also \textbf{make finite-sample estimation more difficult} as its width grows. These challenges motivate our central research question:
\begin{center}
\textit{\textbf{How can we obtain a high-quality subspace for low-dimensional ZO fine-tuning?}}
\end{center}
AGZO constructs perturbation subspaces from current activations, rebuilding them at each step without accumulating activation structure across training \citep{lin2026agzo}. This repeated construction adds computational cost and can leave the subspace sensitive to stochastic variation; using the same subspace for perturbations also couples gradient coverage to estimation variance.

To address these limitations, we propose \method{}, a ZO fine-tuning method based on \textbf{Activation-Informed Subspace Maintenance}. \method{} continuously accumulates activation information across training into a broader width-$K$ maintained subspace. Each perturbation then activates only $k<K$ directions, combining $h$ shared directions with $k-h$ directions sampled from the remaining maintained space. This decouples the amount of gradient-relevant information retained by the maintained subspace from the dimensionality used in each perturbation, allowing low-dimensional perturbations to access different parts of a broader subspace over time. A population-centered one-sided estimator aggregates these perturbations into the parameter update. We evaluate \method{} across 5 LLMs, ranging from 0.6B to 30B parameters, and 11 downstream tasks. \method{} achieves stronger overall performance than existing ZO baselines, including a 1.26 percentage-point improvement over the strongest fully evaluated baseline on OPT-2.7B and a 2.85 percentage-point improvement over MeZO on OPT-30B. Our contributions are as follows:
\begin{itemize}[leftmargin=*]
    \item We propose \method{}, which accumulates forward activation information across training to maintain a broad, evolving perturbation subspace. Our analysis gives a lower bound on its gradient capture, which we also measure empirically.
    \item We decouple maintained and active subspace widths, allowing low-dimensional perturbations to access a broader space. Our analysis identifies when estimates from the active subspace align better with the gradient than estimates from the full maintained subspace.
    \item We comprehensively evaluate \method{} across 5 LLMs and 11 downstream tasks. The results show stronger overall performance than existing ZO baselines, while peak-memory measurements on Qwen3-0.6B show only modest overhead relative to MeZO.
\end{itemize}

\section{Related Work}
\label{sec:related}

\textbf{Zeroth-Order Fine-Tuning.} ZO fine-tuning estimates parameter updates from forward loss evaluations without backpropagation. We consider $\min_w F(w)$, where $F(w)=\mathbb E_\xi[\mathcal L(w;\xi)]$ and $w$ denotes the trainable parameters. For matrix parameters, we use $w=\operatorname{vec}(W)$ and $z=\operatorname{vec}(Z)$; full-model vectors concatenate the corresponding parameter blocks. Given a perturbation $z$ and scale $\epsilon>0$, a standard two-sided estimator for ZO is as follows:
\[
  \widehat g=
  \frac{\mathcal L(w+\epsilon z;\xi)-\mathcal L(w-\epsilon z;\xi)}
       {2\epsilon}\,z.
\]
MeZO samples $z\sim\mathcal N(0,I)$ and regenerates perturbations from random seeds to avoid storing full perturbation vectors \citep{malladi2023mezo}. Subsequent methods improve how perturbations are generated or how gradient estimates are converted into updates.

\textbf{Curvature-Guided Scaling and Coordinate Selection.} These methods adjust perturbations according to coordinate-level information. HiZOO uses $z=\Sigma^{1/2}u$, where $u\sim\mathcal N(0,I)$ and $\Sigma$ is a diagonal inverse-Hessian approximation estimated from forward evaluations; its factored variant HiZOO-L reduces the memory required for this state \citep{zhao2025hizoo}. CurvZO constructs sparse perturbations $z=m\odot u$, sampling the mask $m$ using accumulated curvature-proxy scores and correcting nonuniform sampling through probability reweighting \citep{wang2026curvzo}. Sparse MeZO instead uses parameter magnitudes to select a subset of coordinates for fine-tuning \citep{liu2025sparsemezo}.

\textbf{Structured Perturbations and Random Subspaces.} Another approach reduces the estimation or computational burden through low-rank structure and subspace restriction. LoZO constructs $Z=UV^\top$, reusing the random right factor $V$ across multiple steps while resampling $U$ \citep{chen2025lozo}. SubZero uses $Z=URV^\top$, with periodically refreshed random orthonormal bases $U,V$ and fresh coefficients $R$ \citep{yu2025subzero}. ZO-Muon uses $Z=QA$ with a random orthonormal basis $Q$, then applies Muon-style orthogonalization to the reduced-space estimate before mapping it back to the parameter space \citep{lang2026zomuon}. These methods construct and reuse structured perturbation spaces, but their random bases do not explicitly target gradient-relevant directions.

\textbf{Activation-Informed Subspaces.} AGZO extracts a compact perturbation subspace from each minibatch's activations using power iteration \citep{lin2026agzo}, while ZO-Act reuses a subspace derived from initial activations \citep{dong2026zoact}. AGZO adapts to current activations, but repeated extraction incurs computational cost and does not retain activation structure observed at earlier training steps. ZO-Act avoids repeated extraction, but its fixed subspace cannot track changes in model representations. Both methods restrict perturbations to a truncated activation span, which excludes gradient components outside that span. Widening the span can improve coverage, but also increases basis storage and, when fully activated, estimation dimension.

\method{} accumulates activation information across steps in a width-$K$ maintained subspace, while each perturbation activates only $k<K$ directions: $h$ shared directions and $k-h$ sampled directions. Directions omitted from one perturbation remain available to others, so the maintained subspace can cover more candidate directions without increasing the dimension of every estimate. Additional related work and a design comparison appear in Appendix~\ref{app:additional-related}.

\section{Methodology: AIM-ZO}
\label{sec:method}

\method{} maintains a broad subspace using forward activations accumulated across training. At each step, it selects a smaller active subspace for each perturbation and aggregates the one-sided population evaluations into a parameter update.

\subsection{Notation and setup}
\label{sec:problem-notation}

Let $\mathcal W=\{W_\ell\}_{\ell=1}^{L}$ denote the trainable matrices, with $W_\ell\in\mathbb{R}^{p_\ell\times d_\ell}$ for layer $\ell$. At iteration $t$, we draw a minibatch $\xi_t$ and write $G_{\ell,t}=\nabla_{W_\ell}\mathcal L(\mathcal W_t;\xi_t)$ for the layer gradient used in the analysis. A perturbation population comprises $N$ perturbations whose associated perturbed parameter points are evaluated on the same minibatch to estimate an update.

Throughout, $\ell$, $t$, and $i$ index layers, training iterations, and population members, respectively. For matrices of the same shape, write $\langle A,B\rangle_F=\operatorname{tr}(A^\top B)$. We use $\|A\|_F=\sqrt{\langle A,A\rangle_F}$ for the Frobenius norm and $\|A\|_2$ for the spectral norm. For any column-orthonormal basis $Q$, $P_Q=QQ^\top$ denotes the orthogonal projector onto its column space.

\subsection{Activation-informed subspace maintenance}
\label{sec:space-maintenance}

At iteration $t$, let $H_{\ell,t}\in\mathbb R^{n_{\ell,t}\times d_\ell}$ be the input activations of a linear layer, where $n_{\ell,t}$ counts the activation rows and $Y_{\ell,t}=H_{\ell,t}W_{\ell,t}^\top$. The weight gradient satisfies
\begin{equation}
  G_{\ell,t}=E_{\ell,t}^\top H_{\ell,t},
  \label{eq:activation-gradient}
\end{equation}
where $E_{\ell,t}$ is the loss derivative with respect to $Y_{\ell,t}$. Thus, the row space of $G_{\ell,t}$ lies in the row space of $H_{\ell,t}$ \citep{lin2026agzo}. We maintain an orthonormal basis $Q_{\ell,t}^{\mathrm{wide}}\in\mathbb R^{d_\ell\times K}$ to track prominent activation directions across iterations.

An unperturbed centre pass updates the stored basis by Oja's rule \citep{huang2021streaming}:
\begin{equation}
  Q_{\ell,t}^{\mathrm{wide}}=\operatorname{qf}\!\left(\widetilde Q_{\ell,t}+\eta_q n_{\ell,t}^{-1}H_{\ell,t}^{\top}(H_{\ell,t}\widetilde Q_{\ell,t})\right),
  \label{eq:oja}
\end{equation}
Here $\eta_q$ is the Oja step size, and $\operatorname{qf}$ returns a thin, unpivoted QR basis with fixed signs. A Gaussian matrix initializes the basis. One update serves all $N$ perturbations, after which we store $\widetilde Q_{\ell,t+1}=Q_{\ell,t}^{\mathrm{wide}}$. Appendix~\ref{app:oja-current-comparison} compares gradient capture and construction time with AGZO's current-batch reconstruction \citep{lin2026agzo}; Appendix~\ref{app:maintenance-frequency} reports an online maintenance-frequency ablation.

\begin{algorithm}[t]
  \caption{One training run of \method{}}
  \label{alg:method}
  \begin{algorithmic}[1]
    \Require Iterations $T$; weights $\mathcal W_0$; bases $\{\widetilde Q_{\ell,0}\}$; widths $K,h,k$ with $0\le h\le k\le K$; population size $N\ge2$; schedules $\{\epsilon_t\},\{\eta_t\}$; Oja step size $\eta_q$
    \For{$t=0,\ldots,T-1$}
      \State Sample $\xi_t$; run a centre forward pass to collect $\{H_{\ell,t}\}$
      \State For each layer $\ell$, update $Q_{\ell,t}^{\mathrm{wide}}$ by Equation~\ref{eq:oja}
      \For{$i=1,\ldots,N$}
        \State For each layer $\ell$, sample $S_{\ell,t,i}$, $a_{\ell,t,i}$, and $b_{\ell,t,i}$; form $Q_{\ell,t,i}^{\mathrm{act}}$ and $Z_{\ell,t,i}$ by Equations~\ref{eq:active}--\ref{eq:probe}
        \State Evaluate $y_{t,i}^{+}=\mathcal{L}(\mathcal W_t+\epsilon_t\mathcal Z_{t,i};\xi_t)$ with $\mathcal Z_{t,i}=\{Z_{\ell,t,i}\}_{\ell=1}^{L}$
      \EndFor
      \State For each $\ell$, compute $\widehat G_{\ell,t}$ by Equation~\ref{eq:rloo-estimator} and set $W_{\ell,t+1}=W_{\ell,t}-\eta_t\widehat G_{\ell,t}$
      \State Store $\widetilde Q_{\ell,t+1}\gets Q_{\ell,t}^{\mathrm{wide}}$ for all $\ell$
    \EndFor
    \State \Return $\mathcal W_T$
  \end{algorithmic}
\end{algorithm}

\subsection{Active-subspace selection and parameter updates}
\label{sec:queries-update}

From the width-$K$ subspace maintained across iterations, \method{} selects a width-$k$ active subspace for each perturbation. Different perturbations can thus access different parts of the maintained subspace while keeping each estimate low-dimensional. An online comparison with a shared active subspace is in Appendix~\ref{app:active-resampling}.

For the layer-wise expressions below, we fix an iteration and suppress the layer and iteration indices. For each perturbation, we form an active subspace from the same first $h$ shared basis directions $Q_h=Q^{\mathrm{wide}}[:,1\!:\!h]$ and $s=k-h$ columns sampled from the remaining pool, with $0\le h\le k\le K$ and $k\ge1$. Specifically, perturbation $i$ uses
\begin{equation}
  Q_i^{\mathrm{act}}
  =\left[Q_h,Q^{\mathrm{wide}}[:,S_i]\right],
  \qquad |S_i|=s,
  \label{eq:active}
\end{equation}
where $Q_i^{\mathrm{act}}$ is a basis of the active subspace and $S_i$ is sampled uniformly without replacement from $\{h+1,\ldots,K\}$, independently across population members and layers. The shared columns follow the stored basis order.

Following the low-rank perturbation construction of LoZO \citep{chen2025lozo}, we independently draw $a_i\sim\mathcal N(0,I_p)$ and $b_i\sim\mathcal N(0,I_k)$ and form
\begin{equation}
  U_i=a_i b_i^\top (Q_i^{\mathrm{act}})^\top,\qquad
  \alpha_i=\frac{\sqrt{pd}}{\|U_i\|_F+10^{-12}},\qquad
  Z_i=\alpha_i U_i.
  \label{eq:probe}
\end{equation}
This normalization matches the layer-wise perturbation magnitude to the root-mean-square Frobenius norm of MeZO's isotropic Gaussian perturbations.

For $N\ge2$, each population member perturbs all layers on the same minibatch: $\mathcal Z_i=\{Z_{\ell,i}\}_{\ell=1}^{L}$, $y_i^{+}=\mathcal L(\mathcal W+\epsilon\mathcal Z_i;\xi)$, and $\bar y^{+}=N^{-1}\sum_{i=1}^{N}y_i^{+}$. For each layer, the REINFORCE leave-one-out (RLOO) estimator is
\begin{equation}
  \widehat G
  =\frac1{(N-1)\epsilon}
    \sum_{i=1}^{N}(y_i^{+}-\bar y^{+})Z_i.
  \label{eq:rloo-estimator}
\end{equation}
Each iteration uses one unperturbed and $N$ perturbed forward passes, followed by a parameter update with learning rate $\eta$ (Algorithm~\ref{alg:method}). Further ablations of population size and perturbation coefficient rank are in Appendix~\ref{app:component-ablations}.

\section{Analysis of AIM-ZO Components}
\label{sec:theory}
This section characterizes gradient capture by the activation-maintained subspace and derives a condition under which decoupling maintained and active widths improves directional alignment. It then examines one-sided and two-sided estimation at a matched evaluation budget. Proofs are provided in Appendices~\ref{app:tracking}--\ref{app:usefulness}.

The analysis uses the notation of Section~\ref{sec:problem-notation} and suppresses the layer index: $W\in\R^{p\times d}$ denotes one weight matrix, $f$ its minibatch loss with other weights fixed, and $G=\nabla f(W)$.

\subsection{Maintaining gradient-relevant information}
\paragraph{Gradient structure and capture.} Let $Q_t:=Q_t^{\mathrm{wide}}$ denote the maintained basis. Define $S_t=H_t^\top H_t/n_t$ and $M_t=\E[S_t\mid\mathcal F_{t-1}]$, where $\mathcal F_{t-1}$ is the history before sampling the current batch. Let $V_t^\star$ span the top-$K$ eigenspace of $M_t$. The tracking error between the maintained and target subspaces is
\[
\varepsilon_{{\rm trk},t}:=\|P_{Q_t}-P_{V_t^\star}\|_2.
\]

Let $\mathcal I_t$ index a sliding window of $T_{\rm w}$ training steps with nonzero minibatch gradients $G_\tau$. All gradients in this window are compared against the same $V_t^\star$ and $Q_t$. Prior observations of concentrated and locally stable gradient directions \citep{gurari2018tiny,jaiswal2025stabilization} motivate the following shared-subspace assumption.
\begin{assumption}[Local shared gradient structure]
\label{ass:theory-structure}
For each window $\mathcal I_t$, there exists a shared column-orthonormal basis $U_t^\star\in\R^{d\times r}$, with $1\le r\le K$, such that
\begin{equation}
 \left[\frac1{T_{\rm w}}\sum_{\tau\in\mathcal I_t}
 \frac{\|G_\tau(I-P_{U_t^\star})\|_F^2}{\|G_\tau\|_F^2}\right]^{1/2}
 \le\varepsilon_{{\rm sub},t}.
 \label{eq:theory-structural}
\end{equation}
\end{assumption}
The activation target must also cover the shared gradient subspace.
\begin{assumption}[Activation--gradient alignment]
\label{ass:theory-alignment}
The width-$K$ population activation subspace covers the gradient energy within the shared rank-$r$ space spanned by the same $U_t^\star$ from Assumption~\ref{ass:theory-structure}, up to residual $\varepsilon_{{\rm act},t}$:
\begin{equation}
 \left[\frac1{T_{\rm w}}\sum_{\tau\in\mathcal I_t}
 \frac{\|G_\tau P_{U_t^\star}(I-P_{V_t^\star})\|_F^2}
      {\|G_\tau\|_F^2}\right]^{1/2}
 \le\varepsilon_{{\rm act},t}.
 \label{eq:theory-alignment}
\end{equation}
\end{assumption}

Figure~\ref{fig:structure-diagnostics}(a,b) reports diagnostics for these two residuals: the shared-gradient residual is smaller in middle and late layers, and the activation-alignment residual decreases with $K$.

\paragraph{Tracking the activation subspace.} For the Oja update in Equation~\ref{eq:oja}, define target drift as $\omega_t=\|P_{V_t^\star}-P_{V_{t-1}^\star}\|_F/\sqrt2$ and set $t=0$ after burn-in.
\begin{assumption}[Burn-in initialization]
\label{ass:theory-burn-in}
For $0\le\delta_{\rm init}<1$, the initial error $\Phi_0:=\tfrac12\|P_{Q_0}-P_{V_0^\star}\|_F^2$ is at most $\varepsilon_{\rm burn}^2$ with probability at least $1-\delta_{\rm init}$.
\end{assumption}
Stationary Oja results support this initialization condition under their sampling and step-size requirements \citep{huang2021streaming}. Post-burn-in tracking further requires the following conditions.
\begin{assumption}[Post-burn-in tracking conditions]
\label{ass:theory-tracking}
For $1\le t\le T$, the following conditions hold almost surely: (i) \emph{spectral separation}, $\lambda_K(M_t)-\lambda_{K+1}(M_t)\ge\gamma>0$; (ii) \emph{bounded observations}, $\|S_t\|_2\le\Lambda$ and $\E[\|S_t\|_2^2\mid\mathcal F_{t-1}]\le\Sigma^2$; and (iii) \emph{bounded target drift}, $\omega_t\le\bar\omega$. Here $\lambda_j(M_t)$ is the $j$th largest eigenvalue of $M_t$.
\end{assumption}

\begin{theorem}[Post-burn-in tracking]
\label{thm:theory-track}
Suppose Assumptions~\ref{ass:theory-burn-in} and~\ref{ass:theory-tracking} and the step-size, drift, noise, and burn-in conditions in Appendix~\ref{app:tracking} hold. Then, for $\eta_q>0$, tracking-bound failure probability $0<\delta<1$, and $t\le T$,
\begin{equation}
\E\varepsilon_{{\rm trk},t}^2
\le \underbrace{O\!\left(e^{-\eta_q\gamma t/2}\varepsilon_{\rm burn}^2\right)}_{\text{initialization}}
+\underbrace{O\!\left(\frac{\bar\omega^2}{\eta_q^2\gamma^2}\right)}_{\text{tracking lag}}
+\underbrace{O\!\left(\frac{K\eta_q\Sigma^2}{\gamma}\right)}_{\text{stochastic error}}
+\underbrace{O\!\left(K(\delta+\delta_{\rm init})\right)}_{\text{failure events}}.
\label{eq:theory-track-bound}
\end{equation}
\end{theorem}
A larger eigengap $\gamma$ reduces tracking lag and stochastic error, while larger drift $\bar\omega$ or observation second moment $\Sigma^2$ increases them. Within the admissible range, increasing $\eta_q$ speeds initialization decay and reduces lag but increases stochastic error; the last two terms also scale with $K$. Together with the structural and alignment residuals, this bound gives the following capture guarantee.
\begin{proposition}[Maintained-subspace capture]
\label{prop:theory-capture}
Under Assumptions~\ref{ass:theory-structure}--\ref{ass:theory-alignment},
\begin{equation}
\frac1{T_{\rm w}}\sum_{\tau\in\mathcal I_t}
\frac{\|G_\tau P_{Q_t}\|_F^2}{\|G_\tau\|_F^2}
\ge 1-(\varepsilon_{{\rm sub},t}
+\varepsilon_{{\rm act},t}+\varepsilon_{{\rm trk},t})^2.
\label{eq:theory-capture}
\end{equation}
\end{proposition}
Figure~\ref{fig:structure-diagnostics}(c) shows 52.2\%--95.1\% gradient capture by the maintained subspace.

\begin{figure}[tbp]
  \centering
  \includegraphics[width=\textwidth]{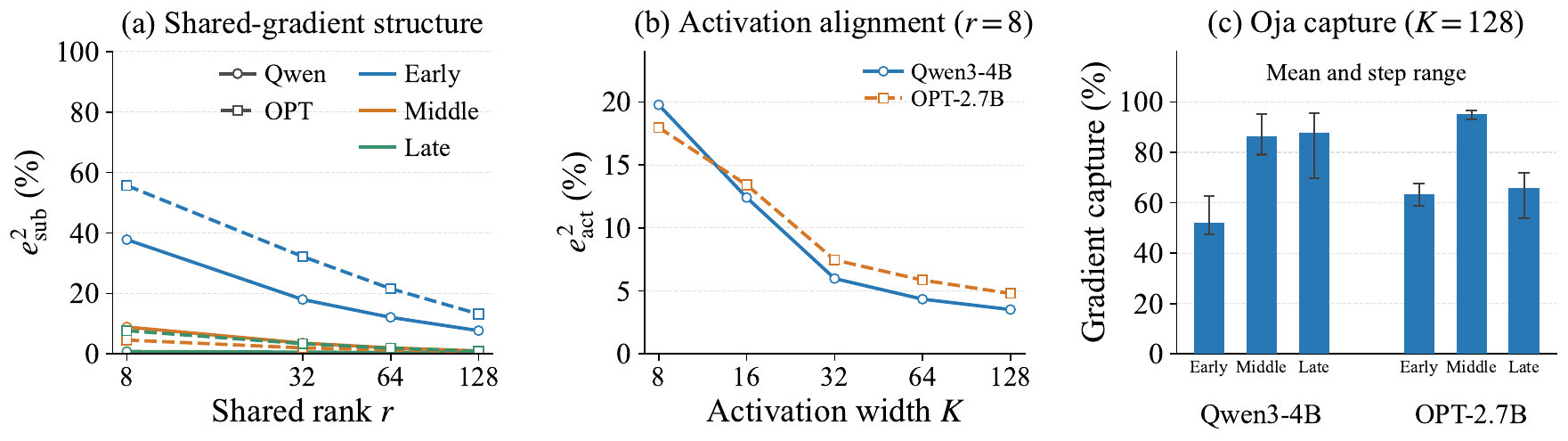}
\caption{RTE diagnostics over 16 consecutive training steps using exact training-set mean gradients. (a) Held-out shared-gradient residual. (b) Activation-alignment residual ($r=8$, averaged over three layers). (c) Maintained-subspace capture ($K=128$); bars show means and whiskers min--max ranges. Protocols are in Appendix~\ref{app:structural-diagnostics}.}
  \label{fig:structure-diagnostics}
\end{figure}

\subsection{Decoupling maintained and active widths}
\label{sec:theory-active}
A wider maintained subspace can improve gradient capture, but using all $K$ basis directions in each estimate increases variance. Fix $G\ne0$ and an ordered orthonormal maintained basis $Q\in\R^{d\times K}$. As in Section~\ref{sec:queries-update}, $Q_S$ uses $h=\theta k\in\mathbb Z$ shared columns and $k-h$ sampled columns, where $0<\theta<1$ and $1\le k\le K$. Define
\[
C_Q=\frac{\|GP_Q\|_F^2}{\|G\|_F^2}>0,\qquad
\tau_h=\frac{\|GP_{Q_h}\|_F^2}{\|GP_Q\|_F^2},\qquad
\rho=\frac{k-h}{K-h}.
\]
For standard Gaussian $R\in\R^{p\times k}$, let $Z=RQ_S^\top$ and $\widehat G_{k,0}=\langle G,Z\rangle_F Z$. Applying the Gaussian subspace moments of \citet[Theorem~2]{yu2025subzero} conditional on $Q_S$ gives
\begin{equation}
\E_R\|\widehat G_{k,0}-G\|_F^2
=\underbrace{(pk+1)\|GP_{Q_S}\|_F^2}_{\text{estimation variance}}
+\underbrace{\|G(I-P_{Q_S})\|_F^2}_{\text{squared projection bias}}.
\label{eq:theory-active-mse}
\end{equation}
Full activation retains all maintained signal at variance factor $pK+1$, while increasing $K$ at fixed $h,k$ lowers tail inclusion $\rho$. The sampling rule and fixed-subspace cosine identity of \citet[Theorem~5.4]{lin2026agzo} yield the following result.

\begin{proposition}[Active capture and directional alignment]
\label{prop:theory-active}
Let $A_{K,k}$ denote the expected fraction of maintained gradient energy retained by the active subspace, $J_{K,k}:=\E\cos(\widehat G_{k,0},G)$, and $\beta_D:=\Gamma(D/2)/[\sqrt\pi\Gamma((D+1)/2)]$, with $\cos(0,G)=0$. Then
\begin{equation}
\frac{\E_S\|GP_{Q_S}\|_F^2}{\|G\|_F^2}=C_Q A_{K,k},\qquad A_{K,k}=\tau_h+\rho(1-\tau_h).
\label{eq:theory-active-result}
\end{equation}
For $k<K$, if $A_{K,k}>\beta_{pK}/\beta_{pk}$, then $J_{K,k}>J_{K,K}$: partial activation has higher expected cosine with $G$ than full activation.
\end{proposition}

Since $\beta_D\simeq\sqrt{2/(\pi D)}$ for large $D$, the sufficient threshold is approximately $\sqrt{k/K}$. Offline capture diagnostics meet this condition in the evaluated settings (Appendix~\ref{app:active-width-diagnostics}). Online scans of maintained and shared widths are in Appendices~\ref{app:width-training}, \ref{app:sampled-maintained-width}, and~\ref{app:shared-width-training}; proofs are in Appendix~\ref{app:usefulness}.

\subsection{One-sided estimation in a high-capture subspace}
\label{sec:theory-op}
Online training and offline ablations favor the RLOO estimator used by \method{} over two-sided estimation (Appendices~\ref{app:online-op-tp} and~\ref{app:offline-population}). To examine how subspace quality affects this advantage, we vary gradient capture at fixed active width and compare the directional alignment of the two estimators through $\Delta_{\cos}=\cos(\widehat G_{\rm OP},G)-\cos(\widehat G_{\rm TP},G)$.

Figure~\ref{fig:capture-op-tp} shows that the OP--TP alignment gain increases with gradient capture. Actual-loss results closely follow a linear control constructed from the same-batch gradient, indicating that the trend is largely explained by the first-order signal. Further results are in Appendix~\ref{app:space-quality}.

\begin{figure}[tbp]
  \centering
  \includegraphics[width=\textwidth]{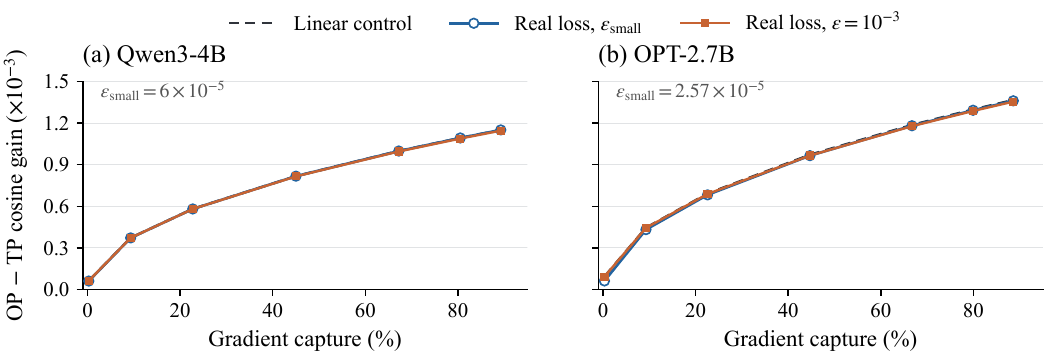}
\caption{OP--TP cosine gain versus gradient capture on RTE. OP16 and TP8 each use 16 perturbed loss evaluations at $k=64$. Points average matched populations across checkpoints and batches; dashed curves show the corresponding linear control.}
  \label{fig:capture-op-tp}
\end{figure}

\section{Experiments}
\label{sec:experiments}

\newcommand{\expscore}[2]{$#1_{\pm#2}$}

In this section, we empirically evaluate the proposed AIM-ZO methods over a wide collection of datasets and LLMs. Section \ref{5.1} compares the performance of AIM-ZO on two mainly chosen LLMs for ZO. Section \ref{5.3} evaluates the peak GPU memory consumption. Additional experiments are in the Appendix: Appendix~\ref{app:structural-diagnostics} complements the analysis in Section~\ref{sec:theory} and Appendix~\ref{app:additional-finetuning} provides broader fine-tuning results, including all 11 OPT-2.7B tasks and additional evaluations on Qwen3 LLMs.

\textbf{Experiment Settings.} We consider a wide collection of LLMs, including OPT-2.7B, OPT-13B, OPT-30B, and Qwen3-8B-Base; supplementary experiments consider two more LLMs, Qwen3-0.6B-Base and Qwen3-4B. We fine-tune LLMs over the 11 downstream tasks included in \citet{malladi2023mezo}; OPT-2.7B covers all 11 tasks, while the main tables report six-task comparisons. 

\textbf{Baselines.} We include various baselines, including \textbf{1)} zero-shot scores provide a non-training reference. \textbf{2)} MeZO \cite{malladi2023mezo} with full perturbation. \textbf{3)} HiZOO \citep{zhao2025hizoo} and CurvZO \citep{wang2026curvzo} with curvature-guided scaling, and \textbf{4)} LoZO \citep{chen2025lozo}, AGZO \citep{lin2026agzo}, and ZO-Muon \citep{lang2026zomuon} as subspace-based perturbation baselines.

\textbf{Implementation details }All fine-tuning methods use BF16 and approximately 40,000 training forward evaluations, with update counts adjusted for each method's forward evaluations per update. \method{} uses 16 forward evaluations per update and thus runs for 2,500 updates, with $K=128$, $h=48$, and $k=64$. Most results use three seeds; tables report mean percentages and available sample standard deviations. Data splits and evaluation protocols are in Appendix~\ref{app:experimental-setup}; component ablations are in Appendix~\ref{app:component-ablations}.

\begin{table}[tbp]
\centering\vspace{-1em}
\caption{Multi-method comparison on OPT-2.7B and OPT-13B. Scores are accuracy (\%) except SQuAD (F1). Avg. is the unweighted mean over six tasks. Best mean scores are bold; subscripts report available sample standard deviations. Evaluation details are in Appendix~\ref{app:additional-finetuning}.}
\label{tab:opt_multimethod}
\setlength{\tabcolsep}{5pt}
\resizebox{\linewidth}{!}{%
\begin{tabular}{lccccccc}
\toprule
Method & RTE & BoolQ & SST-2 & WiC & WSC & SQuAD & Avg. \\
\midrule
\multicolumn{8}{c}{OPT-2.7B} \\
\midrule
Zero-shot & 55.23 & 52.87 & 56.65 & 54.86 & 36.54 & 26.92 & 47.18 \\
MeZO & \expscore{65.13}{1.67} & \expscore{66.10}{2.26} & \expscore{92.50}{0.54} & \expscore{58.53}{0.41} & \expscore{54.81}{1.52} & \expscore{80.99}{1.42} & 69.68 \\
CurvZO & \expscore{64.53}{4.83} & \expscore{68.01}{0.62} & \expscore{\mathbf{92.96}}{0.31} & \expscore{57.96}{2.07} & \expscore{46.47}{7.28} & \expscore{58.71}{2.47} & 64.77 \\
HiZOO & \expscore{60.29}{2.94} & \expscore{67.31}{1.34} & \expscore{91.77}{0.67} & \expscore{58.88}{0.80} & \expscore{53.53}{2.42} & \expscore{73.81}{0.90} & 67.60 \\
AGZO & \expscore{64.40}{3.10} & \expscore{66.26}{1.46} & \expscore{92.59}{0.51} & \expscore{57.02}{1.73} & \expscore{47.76}{2.22} & \expscore{29.82}{5.54} & 59.64 \\
ZO-Muon & \expscore{63.32}{2.53} & \expscore{\mathbf{68.90}}{1.12} & \expscore{92.75}{0.52} & \expscore{\mathbf{61.13}}{1.12} & \expscore{51.60}{3.38} & \expscore{78.93}{1.13} & 69.44 \\
LoZO & \expscore{61.88}{4.48} & \expscore{65.33}{0.73} & \expscore{91.51}{0.83} & \expscore{53.67}{2.70} & \expscore{50.96}{7.26} & \expscore{43.52}{1.24} & 61.15 \\
\textsc{AIM-ZO (Ours)} & \expscore{\mathbf{67.51}}{3.18} & \expscore{67.22}{1.67} & \expscore{92.87}{0.47} & \expscore{60.13}{1.70} & \expscore{\mathbf{56.73}}{4.35} & \expscore{\mathbf{81.19}}{0.49} & $\mathbf{70.94}$ \\
\midrule
\multicolumn{8}{c}{OPT-13B} \\
\midrule
Zero-shot & 61.73 & 60.46 & 60.21 & 55.17 & 40.38 & 40.80 & 53.13 \\
MeZO & \expscore{67.36}{1.54} & \expscore{71.08}{1.37} & \expscore{92.00}{0.69} & \expscore{58.58}{1.93} & \expscore{54.49}{8.72} & \expscore{58.64}{1.42} & 67.02 \\
CurvZO & \expscore{52.23}{2.73} & \expscore{61.73}{0.42} & \expscore{49.39}{0.78} & \expscore{49.01}{1.07} & \expscore{\mathbf{61.22}}{3.89} & \expscore{65.56}{1.42} & 56.52 \\
HiZOO & \expscore{\mathbf{73.65}}{1.65} & \expscore{71.58}{0.12} & \expscore{91.63}{1.71} & \expscore{\mathbf{59.82}}{0.59} & \expscore{58.01}{1.11} & \expscore{75.45}{0.74} & 71.69 \\
AGZO & \expscore{52.59}{5.76} & \expscore{61.43}{1.10} & \expscore{88.23}{0.98} & \expscore{49.43}{2.20} & \expscore{57.37}{2.00} & \expscore{75.24}{1.24} & 64.05 \\
ZO-Muon & \expscore{63.06}{1.50} & \expscore{66.53}{1.39} & \expscore{91.78}{0.65} & \expscore{56.79}{0.94} & \expscore{50.00}{5.85} & \expscore{80.35}{0.93} & 68.08 \\
LoZO & \expscore{49.70}{2.80} & \expscore{62.15}{0.07} & \expscore{50.27}{0.52} & \expscore{47.39}{1.01} & \expscore{57.37}{5.88} & \expscore{47.40}{6.68} & 52.38 \\
\textsc{AIM-ZO (Ours)} & \expscore{70.69}{2.97} & \expscore{\mathbf{71.93}}{2.65} & \expscore{\mathbf{93.78}}{0.39} & \expscore{59.52}{0.39} & \expscore{57.05}{2.42} & \expscore{\mathbf{81.58}}{0.25} & $\mathbf{72.42}$ \\
\bottomrule
\end{tabular}}\vspace{-1em}
\end{table}

\subsection{Comparison with zeroth-order methods}\label{5.1}

Table~\ref{tab:opt_multimethod} compares \method{} with six representative ZO baselines on OPT-2.7B and OPT-13B under matched forward-evaluation budgets. All methods are evaluated using our unified protocol, which may differ from the settings reported in their original papers.

\textbf{Observation 1: AIM-ZO achieves the strongest overall performance across tasks.}
\method{} obtains the highest six-task average among methods with complete results on both OPT-2.7B (70.94\%) and OPT-13B (72.42\%). Relative to MeZO, it improves all six reported tasks at both model scales, and further achieves higher mean performance on nine of the full 11 OPT-2.7B tasks (Appendix~\ref{app:additional-finetuning}). Although several baselines remain stronger on individual tasks, \method{} avoids the large task-specific degradations observed for some alternatives and therefore provides more balanced performance across the benchmark.

\textbf{Observation 2: AIM-ZO consistently improves over existing subspace-based ZO methods.}
Compared with representative subspace methods, including LoZO, AGZO, and ZO-Muon, \method{} achieves the highest six-task average at both scales. On OPT-2.7B, it improves the average from 69.44\% for the strongest subspace baseline, ZO-Muon, to 70.94\%; on OPT-13B, the corresponding gap increases from 68.08\% to 72.42\%. The advantage is particularly pronounced over AGZO and LoZO, whose compact perturbation spaces show substantial degradation on several tasks. These results suggest that the quality and maintenance of the perturbation subspace matter beyond simply restricting ZO updates to a low-dimensional space.
 
Besides OPT-2.7B and OPT-13B, Table~\ref{tab:scaling_results} extends the comparison to OPT-30B and Qwen3-8B. On OPT-30B, \method{} improves over MeZO on all six tasks, raising the average by 2.85 percentage points. On Qwen3-8B, it improves on five of six tasks and raises the average by 0.93 points over MeZO; SQuAD has the largest gain, while SST-2 declines. These results show stronger gains as OPT scales and positive gains on Qwen3.

\begin{table}[tbp]
\centering\vspace{-1em}
\caption{Scaling and model-family comparisons on OPT-30B and Qwen3-8B (\%). SQuAD uses F1; other tasks use accuracy. Avg. is the unweighted mean over six tasks. Best mean scores are bold; subscripts report standard deviations. Qwen3-8B entries use three seeds and development-selected checkpoints.}
\label{tab:scaling_results}
\label{tab:opt30_results}
\setlength{\tabcolsep}{5pt}
\resizebox{\linewidth}{!}{%
\begin{tabular}{lccccccc}
\toprule
\midrule
\multicolumn{8}{c}{OPT-30B} \\
\midrule
Method & RTE & BoolQ & SST-2 & WiC & WSC & SQuAD & Avg. \\
\midrule
Zero-shot & 53.79 & 40.90 & 58.83 & 52.98 & 37.50 & 46.48 & 48.41 \\
MeZO & \expscore{63.25}{1.79} & \expscore{70.74}{1.17} & \expscore{91.08}{0.36} & \expscore{55.96}{1.81} & \expscore{58.08}{2.85} & \expscore{80.22}{0.99} & 69.89 \\
\textsc{AIM-ZO (Ours)} & \expscore{\mathbf{68.16}}{2.66} & \expscore{\mathbf{74.90}}{1.53} & \expscore{\mathbf{93.47}}{0.56} & \expscore{\mathbf{57.71}}{0.91} & \expscore{\mathbf{59.23}}{2.77} & \expscore{\mathbf{82.96}}{0.69} & $\mathbf{72.74}$ \\
\midrule
\multicolumn{8}{c}{Qwen3-8B-Base} \\
\midrule
Method & RTE & BoolQ & SST-2 & WiC & MultiRC & SQuAD & Avg. \\
\midrule
Zero-shot & 86.28 & 74.60 & 58.49 & 65.67 & 72.77 & 83.67 & 73.58 \\
MeZO & \expscore{88.33}{1.37} & \expscore{85.75}{0.36} & \expscore{\mathbf{90.94}}{0.40} & \expscore{67.92}{1.31} & \expscore{86.01}{0.48} & \expscore{86.90}{1.95} & 84.31 \\
AGZO & \expscore{86.88}{1.50} & \expscore{79.40}{0.85} & \expscore{81.42}{2.88} & \expscore{65.78}{1.75} & \expscore{83.76}{1.88} & \expscore{68.20}{4.77} & 77.57 \\
\textsc{AIM-ZO (Ours)} & \expscore{\mathbf{89.53}}{1.08} & \expscore{\mathbf{85.76}}{0.13} & \expscore{89.49}{0.81} & \expscore{\mathbf{68.18}}{1.51} & \expscore{\mathbf{86.68}}{0.16} & \expscore{\mathbf{91.80}}{0.42} & $\mathbf{85.24}$ \\
\bottomrule
\end{tabular}}
\end{table}

\subsection{Ablation Studies}

We ablate the three main design choices of \method{}. First, \textbf{online subspace maintenance} is beneficial: on OPT-2.7B RTE, updating the maintained subspace every step achieves higher mean performance than either a fixed activation subspace or less frequent updates (Appendix~\ref{app:component-ablations}). Second, \textbf{per-perturbation active-subspace resampling} consistently improves over using one common active subspace for the whole population across all four tested tasks, supporting the benefit of exposing different perturbations to different parts of the maintained space (Appendix~\ref{app:active-resampling}). Third, the \textbf{population-centered one-sided estimator} outperforms the two-sided alternative at both the best and final checkpoints on OPT-2.7B RTE (Appendix~\ref{app:online-op-tp}). Additional width, population-size, and perturbation-rank ablations are reported in Appendix~\ref{app:component-ablations}.

\subsection{Peak GPU memory}\label{5.3}
Figure~\ref{fig:drop-memory} compares peak memory during complete training updates. On Qwen3-0.6B-Base/DROP, \method{} uses 0.242--0.254\,GiB more than MeZO and less than AGZO and full-parameter SGD across the measured configurations. SGD runs out of memory at the largest tested tensor lengths or batch sizes, and AGZO at batch size 64. Measurement details, runtime comparisons, and caching ablations are in Appendix~\ref{app:runtime-memory}.

\begin{figure}[t]
\centering
\includegraphics[width=\textwidth]{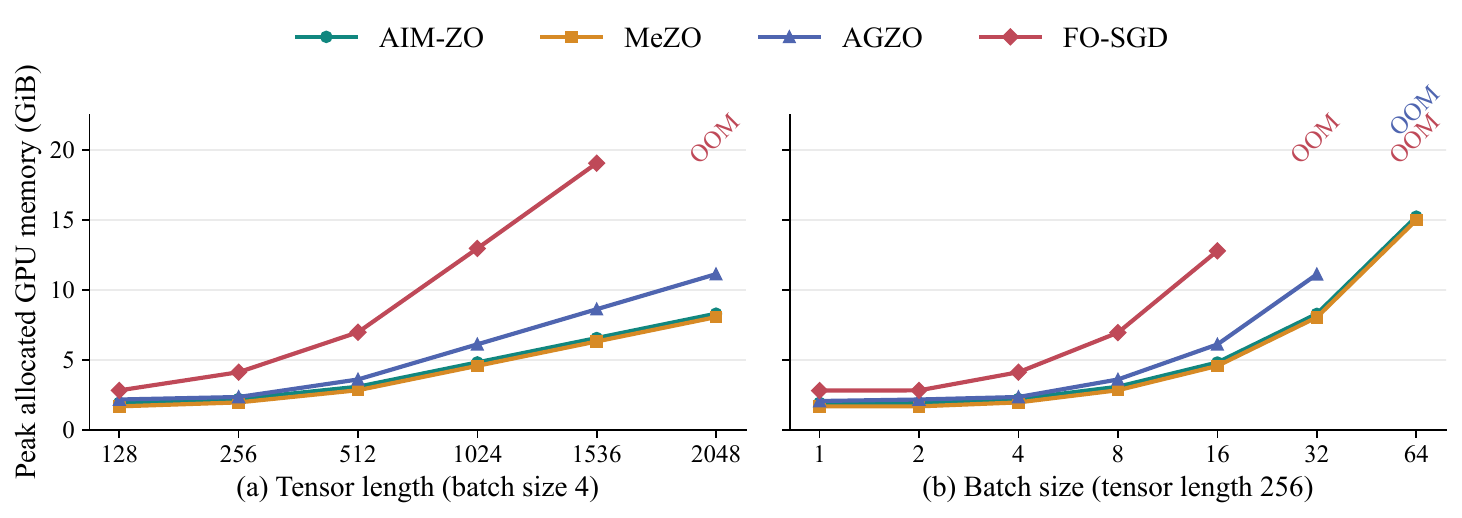}
\caption{Peak allocated GPU memory on Qwen3-0.6B-Base/DROP with BF16 weights. Left: varying tensor length at batch size four. Right: varying batch size at tensor length 256. OOM marks observed failures, not measured memory values.}
\label{fig:drop-memory}
\end{figure}

\section{Conclusion}
\label{sec:conclusion}

We presented \method{}, which maintains an activation-informed subspace across training and selects smaller shared-and-sampled active subspaces for forward-only ZO updates. Our analysis bounds gradient capture by the maintained subspace and characterizes when the active subspace improves directional alignment over full activation. Across five models and 11 tasks, \method{} attains the best six-task average among fully evaluated ZO baselines on OPT-2.7B and improves over MeZO on all six OPT-30B tasks. Gains on Qwen3 are smaller, with modest peak-memory overhead.

The analysis assumes post-burn-in tracking conditions and Gaussian perturbations, leaving the implemented rank-one updates and changing active subspaces only partially characterized. Adapting the preset widths and tail sampling using loss feedback remains future work.

\clearpage
\bibliography{references}

@inproceedings{malladi2023mezo,
  title     = {Fine-Tuning Language Models with Just Forward Passes},
  author    = {Malladi, Sadhika and Gao, Tianyu and Nichani, Eshaan and Damian, Alex and Lee, Jason D. and Chen, Danqi and Arora, Sanjeev},
  booktitle = {Advances in Neural Information Processing Systems},
  volume    = {36},
  pages     = {53038--53075},
  year      = {2023}
}

@inproceedings{zhao2025hizoo,
  title     = {Second-Order Fine-Tuning without Pain for {LLM}s: A Hessian Informed Zeroth-Order Optimizer},
  author    = {Zhao, Yanjun and Dang, Sizhe and Ye, Haishan and Dai, Guang and Qian, Yi and Tsang, Ivor W.},
  booktitle = {International Conference on Learning Representations},
  year      = {2025}
}

@article{wang2026curvzo,
  title   = {{CurvZO}: Adaptive Curvature-Guided Sparse Zeroth-Order Optimization for Efficient {LLM} Fine-Tuning},
  author  = {Wang, Shuo and Chen, Ziyu and Tang, Ming},
  journal = {arXiv preprint arXiv:2603.21725},
  year    = {2026}
}

@article{lang2026zomuon,
  title   = {Powering Up Zeroth-Order Training via Subspace Gradient Orthogonalization},
  author  = {Lang, Yicheng and Wang, Changsheng and Zhang, Yihua and Hong, Mingyi and Zhang, Zheng and Yin, Wotao and Liu, Sijia},
  journal = {arXiv preprint arXiv:2602.17155},
  year    = {2026}
}

@article{ghadimi2013stochastic,
  title   = {Stochastic First- and Zeroth-Order Methods for Nonconvex Stochastic Programming},
  author  = {Ghadimi, Saeed and Lan, Guanghui},
  journal = {SIAM Journal on Optimization},
  volume  = {23},
  number  = {4},
  pages   = {2341--2368},
  year    = {2013},
  doi     = {10.1137/120880811}
}

@article{nesterov2017random,
  title   = {Random Gradient-Free Minimization of Convex Functions},
  author  = {Nesterov, Yurii and Spokoiny, Vladimir},
  journal = {Foundations of Computational Mathematics},
  volume  = {17},
  number  = {2},
  pages   = {527--566},
  year    = {2017},
  doi     = {10.1007/s10208-015-9296-2}
}

@inproceedings{zhao2024galore,
  title     = {{GaLore}: Memory-Efficient {LLM} Training by Gradient Low-Rank Projection},
  author    = {Zhao, Jiawei and Zhang, Zhenyu and Chen, Beidi and Wang, Zhangyang and Anandkumar, Anima and Tian, Yuandong},
  booktitle = {Proceedings of the 41st International Conference on Machine Learning},
  volume    = {235},
  pages     = {61121--61143},
  year      = {2024},
  publisher = {PMLR}
}

@inproceedings{chen2025lozo,
  title     = {Enhancing Zeroth-order Fine-tuning for Language Models with Low-rank Structures},
  author    = {Chen, Yiming and Zhang, Yuan and Cao, Liyuan and Yuan, Kun and Wen, Zaiwen},
  booktitle = {International Conference on Learning Representations},
  year      = {2025}
}

@inproceedings{yu2025subzero,
  title     = {Zeroth-Order Fine-Tuning of {LLM}s in Random Subspaces},
  author    = {Yu, Ziming and Zhou, Pan and Wang, Sike and Li, Jia and Tian, Mi and Huang, Hua},
  booktitle = {Proceedings of the IEEE/CVF International Conference on Computer Vision},
  pages     = {4475--4485},
  year      = {2025}
}

@inproceedings{choromanski2019asebo,
  title     = {From Complexity to Simplicity: Adaptive {ES}-Active Subspaces for Blackbox Optimization},
  author    = {Choromanski, Krzysztof M. and Pacchiano, Aldo and Parker-Holder, Jack and Tang, Yunhao and Sindhwani, Vikas},
  booktitle = {Advances in Neural Information Processing Systems},
  volume    = {32},
  year      = {2019}
}

@inproceedings{lin2026agzo,
  title     = {{AGZO}: Activation-Guided Zeroth-Order Optimization for {LLM} Fine-Tuning},
  author    = {Lin, Wei and Jiang, Yining and Song, Qingyu and Xiang, Qiao and Xu, Hong},
  booktitle = {Proceedings of the International Conference on Machine Learning},
  year      = {2026},
  url       = {https://arxiv.org/abs/2601.17261}
}

@article{dong2026zoact,
  title   = {{ZO-Act}: Efficient Zeroth-Order Fine-Tuning via One-Shot Activation-Informed Low-Rank Subspaces},
  author  = {Dong, Xun and Xu, Yibo and Wang, Naigang and Li, Xin and Yin, Penghang and Yang, Zi},
  journal = {arXiv preprint arXiv:2607.01125},
  year    = {2026}
}

@inproceedings{allenzhu2017oja,
  title     = {First Efficient Convergence for Streaming $k$-{PCA}: A Global, Gap-Free, and Near-Optimal Rate},
  author    = {Allen-Zhu, Zeyuan and Li, Yuanzhi},
  booktitle = {Proceedings of the 58th Annual IEEE Symposium on Foundations of Computer Science},
  pages     = {487--492},
  year      = {2017},
  doi       = {10.1109/FOCS.2017.51}
}

@inproceedings{huang2021streaming,
  title     = {Streaming $k$-{PCA}: Efficient Guarantees for {Oja}'s Algorithm, beyond Rank-One Updates},
  author    = {Huang, De and Niles-Weed, Jonathan and Ward, Rachel},
  booktitle = {Proceedings of the 34th Conference on Learning Theory},
  volume    = {134},
  pages     = {2463--2498},
  year      = {2021},
  publisher = {PMLR}
}

@inproceedings{bienstock2022robust,
  title     = {Robust Streaming {PCA}},
  author    = {Bienstock, Daniel and Jeong, Minchan and Shukla, Apurv and Yun, Se-Young},
  booktitle = {Advances in Neural Information Processing Systems},
  volume    = {35},
  pages     = {4231--4243},
  year      = {2022}
}

@inproceedings{liu2025sparsemezo,
  title     = {Sparse Me{ZO}: Less Parameters for Better Performance in Zeroth-Order {LLM} Fine-Tuning},
  author    = {Liu, Yong and Zhu, Zirui and Gong, Chaoyu and Cheng, Minhao and Hsieh, Cho-Jui and You, Yang},
  booktitle = {Advances in Neural Information Processing Systems},
  volume    = {38},
  year      = {2025}
}

@inproceedings{guo2025staticsparsity,
  title     = {Zeroth-Order Fine-Tuning of {LLM}s with Transferable Static Sparsity},
  author    = {Guo, Wentao and Long, Jikai and Zeng, Yimeng and Liu, Zirui and Yang, Xinyu and Ran, Yide and Gardner, Jacob R. and Bastani, Osbert and De Sa, Christopher and Yu, Xiaodong and Chen, Beidi and Xu, Zhaozhuo},
  booktitle = {International Conference on Learning Representations},
  year      = {2025}
}

@article{gurari2018tiny,
  title   = {Gradient Descent Happens in a Tiny Subspace},
  author  = {Gur-Ari, Guy and Roberts, Daniel A. and Dyer, Ethan},
  journal = {arXiv preprint arXiv:1812.04754},
  year    = {2018}
}

@inproceedings{jaiswal2025stabilization,
  title     = {From Low Rank Gradient Subspace Stabilization to Low-Rank Weights: Observations, Theories, and Applications},
  author    = {Jaiswal, Ajay Kumar and Wang, Yifan and Yin, Lu and Liu, Shiwei and Chen, Runjin and Zhao, Jiawei and Grama, Ananth and Tian, Yuandong and Wang, Zhangyang},
  booktitle = {Proceedings of the 42nd International Conference on Machine Learning},
  volume    = {267},
  pages     = {26740--26756},
  year      = {2025},
  publisher = {PMLR}
}

@inproceedings{liang2024osd,
  title     = {Memory-Efficient {LLM} Training with Online Subspace Descent},
  author    = {Liang, Kaizhao and Liu, Bo and Chen, Lizhang and Liu, Qiang},
  booktitle = {Advances in Neural Information Processing Systems},
  volume    = {37},
  year      = {2024},
  doi       = {10.52202/079017-2054}
}

@inproceedings{rajabi2025subtrack,
  title     = {{SubTrack++}: Gradient Subspace Tracking for Scalable {LLM} Training},
  author    = {Rajabi, Sahar and Nonta, Nayeema and Rambhatla, Sirisha},
  booktitle = {Advances in Neural Information Processing Systems},
  volume    = {38},
  year      = {2025},
  doi       = {10.52202/085713-1060}
}

@article{tropp2011freedman,
  title = {Freedman's Inequality for Matrix Martingales},
  author = {Tropp, Joel A.},
  journal = {Electronic Communications in Probability},
  volume = {16},
  pages = {262--270},
  year = {2011},
  url = {https://tropp.caltech.edu/papers/Tro11-Freedmans-Inequality.pdf}
}

@article{zheng2026agenticesopt,
  title = {Agentic {ESOpt}: Fine-Tuning Long-Horizon {LLM} Agents with Minimal {GPU} Requirements},
  author = {Zheng, Zhi and Chen, Rongsheng and Ba, Yunpeng and Wang, Zhenkun and Teh, Yee Whye and Lee, Wee Sun},
  journal = {arXiv preprint arXiv:2608.17310},
  year = {2026},
  url = {https://arxiv.org/abs/2608.17310}
}

@article{ba2026evolutionstrategies,
  title = {Understanding Evolution Strategies for {LLM} Reasoning: Broader Reasoning Coverage than {GRPO}},
  author = {Ba, Yunpeng and Zheng, Zhi and Xie, Yue and Li, Jiaqing and Tong, Xialiang and Zhong, Tao and Yuan, Mingxuan and Lu, Zhichao and Wu, Xuyang and Wang, Zhenkun},
  journal = {arXiv preprint arXiv:2608.27351},
  year = {2026},
  url = {https://arxiv.org/abs/2608.27351}
}
\bibliographystyle{iclr2027_conference}

\clearpage
\appendix
\section*{Appendix Contents}
\begingroup
\renewcommand{\arraystretch}{1.35}
\noindent\begin{tabular*}{\textwidth}{@{}p{0.88\textwidth}@{\extracolsep{\fill}}r@{}}
\hyperref[app:key-notation]{Key Notation} & \pageref*{app:key-notation} \\
\hyperref[app:tracking]{A. Maintained-Subspace Tracking and Gradient Capture} & \pageref*{app:tracking} \\
\hyperref[app:usefulness]{B. Energy Capture and Directional Alignment in Active Subspaces} & \pageref*{app:usefulness} \\
\hyperref[app:experimental-setup]{C. Experimental Setup and Evaluation Protocol} & \pageref*{app:experimental-setup} \\
\hyperref[app:structural-diagnostics]{D. Supplementary Experiments for Section~\ref*{sec:theory}} & \pageref*{app:structural-diagnostics} \\
\hyperref[app:component-ablations]{E. Active-Subspace Ablations} & \pageref*{app:component-ablations} \\
\hyperref[app:population-quality]{F. Space Quality and One-Sided Estimation} & \pageref*{app:population-quality} \\
\hyperref[app:runtime-memory]{G. Runtime and Peak Memory} & \pageref*{app:runtime-memory} \\
\hyperref[app:additional-finetuning]{H. Additional Fine-Tuning Results} & \pageref*{app:additional-finetuning} \\
\hyperref[app:additional-related]{I. Additional Related Work} & \pageref*{app:additional-related} \\
\end{tabular*}
\endgroup

\phantomsection
\newpage
\section*{Key Notation}
\label{app:key-notation}
We summarize the main symbols used in the analysis; other quantities are defined where they first appear. Each basis matrix represents the subspace spanned by its columns, with $P_Q=QQ^\top$ denoting the corresponding projector. The unindexed $Q$ denotes the maintained basis in Section~\ref{sec:theory-active} and Appendix~\ref{app:usefulness}.

\begingroup
\small
\renewcommand{\arraystretch}{1.22}
\begin{tabular}{@{}>{\raggedright\arraybackslash}p{0.20\textwidth}@{\hspace{0.02\textwidth}}>{\raggedright\arraybackslash}p{0.78\textwidth}@{}}
\toprule
\textbf{Symbol} & \textbf{Description} \\
\midrule
$K,k,h,s$ & Maintained, active, prefix, and sampled-tail widths; $k=h+s\le K$ \\
$Q_t,Q_S$ & Maintained basis at step $t$ and active basis selected from a fixed maintained basis \\
$U_t^\star,V_t^\star$ & Shared gradient basis and population activation target basis \\
$r,\mathcal I_t,T_{\rm w}$ & Shared-gradient rank, sliding window, and window length \\
$S_t,M_t$ & Uncentered activation Gram observation $H_t^\top H_t/n_t$ and its conditional mean \\
$\varepsilon_{{\rm sub},t}$ & Bound on the shared-gradient residual \\
$\varepsilon_{{\rm act},t}$ & Bound on the activation-alignment residual \\
$\varepsilon_{{\rm trk},t}$ & Spectral distance between maintained and target projectors \\
$\Phi_t$ & Squared chordal tracking error \\
$\omega_t,\bar\omega$ & Target-subspace drift and its upper bound \\
$\gamma,\eta_q$ & Eigengap lower bound and Oja step size \\
$C_Q$ & Gradient capture $\|GP_Q\|_F^2/\|G\|_F^2$ \\
$\tau_h,\rho$ & Prefix share of maintained energy and tail inclusion probability \\
$\theta$ & Fixed shared fraction in the active-width comparison; $h=\theta k$ \\
$\beta_D,\mathcal L_{K,k}$ & Gaussian direction factor and shared-and-sampled alignment lower bound \\
$R,D$ & Gaussian coefficient matrix in the active-width analysis and its dimension $D=pk$ \\
$\widehat G_{\rm OP},\widehat G_{\rm TP}$ & One-sided RLOO and two-sided gradient estimators \\
\bottomrule
\end{tabular}
\endgroup

\newpage
\section{Maintained-Subspace Tracking and Gradient Capture}
\label{app:tracking}
\subsection{Conditions and statement}
Let $t=0$ mark the end of burn-in and $T\ge1$ the number of subsequent Oja updates. Define $P_t=Q_tQ_t^\top$, $P_{\star,t}=V_t^\star V_t^{\star\top}$, and
\[
 \Phi_t=K-\operatorname{tr}(P_{\star,t}P_t)
 =\tfrac12\|P_t-P_{\star,t}\|_F^2,\qquad
 \varepsilon_{{\rm trk},t}^2\le\Phi_t\le K .
\]
The history $\mathcal F_{t-1}$ precedes the current observation $S_t$; $Q_{t-1},M_t,P_{\star,t}$ are measurable with respect to this history. We impose the following conditions for a fixed horizon $T$.
\begin{enumerate}
\item There is an $\mathcal F_0$-measurable event $\mathcal E_0=\{\Phi_0\le\varepsilon_{\rm burn}^2\}$ with probability at least $1-\delta_{\rm init}$.
\item $S_t\succeq0$, $M_t=\E[S_t\mid\mathcal F_{t-1}]$, and $\lambda_K(M_t)-\lambda_{K+1}(M_t)\ge\gamma>0$.
\item Almost surely $\|S_t\|_2\le\Lambda$ and $\E[\|S_t\|_2^2\mid\mathcal F_{t-1}]\le\Sigma^2$.
\item $\omega_t=\|P_{\star,t}-P_{\star,t-1}\|_F/\sqrt2\le\bar\omega$.
\end{enumerate}
\paragraph{Initialization.} Under the stationary sampling and two-phase step-size conditions of \citet{huang2021streaming}, Gaussian-initialized Oja attains a prescribed projector accuracy with high probability. In particular, $\|P_{Q_0}-P_{V_0^\star}\|_2\le\varepsilon_{\rm burn}/\sqrt K$ implies $\Phi_0\le\varepsilon_{\rm burn}^2$, since $\Phi_0\le K\|P_{Q_0}-P_{V_0^\star}\|_2^2$. For the evolving activation sequence, we assume this entrance condition.

Take $\eta_q>0$, $0\le\delta_{\rm init}<1$, and $0<\Sigma\le\Lambda$. Set $a=\gamma/2$, $\kappa=1-\eta_q a$ and, for $0<\delta<1$,
\[
 b_T(\delta)=2\sqrt2K\sqrt{\frac{\eta_q\Sigma^2}{a}\log\frac T\delta}
       +\frac83K\eta_q\Lambda\log\frac T\delta .
\]
The explicit sufficient step-size and entrance budgets are
\begin{equation}
 \eta_q\Lambda\le\tfrac18,\qquad
 \varepsilon_{\rm burn}^2\le\tfrac1{16},\qquad
 \bar\omega\le\tfrac{\eta_q a}{4},\qquad
 \frac{16K\eta_q\Lambda^2}{a}+b_T(\delta)\le\tfrac18 .
 \label{eq:app-tracking-budget}
\end{equation}
Since $0\preceq M_t\preceq\Lambda I$, we have $\gamma\le\Lambda$, so they imply $0<\eta_q a\le1/16$. Define
\[
 B_{\max}=\frac{\bar\omega^2}{\eta_q^2a^2}
             +\frac{16K\eta_q\Lambda^2}{a},\qquad
 r_{\rm safe}^2=\max\{\varepsilon_{\rm burn}^2,B_{\max}\}+b_T(\delta).
\]
We prove that, with probability at least $1-\delta_{\rm init}-\delta$, $\Phi_t\le r_{\rm safe}^2$ simultaneously for $t\le T$. The explicit expected bound underlying Theorem~\ref{thm:theory-track} is
\begin{equation}
\E\varepsilon_{{\rm trk},t}^2
\le \kappa^t\varepsilon_{\rm burn}^2+
(1-\kappa^t)\left[
\frac{4\bar\omega^2}{\eta_q^2\gamma^2}
+\frac{32K\eta_q\Sigma^2}{\gamma}\right]
+K(\delta+\delta_{\rm init}).
\label{eq:app-tracking-explicit}
\end{equation}
Using $\kappa^t\le e^{-\eta_q\gamma t/2}$ and $1-\kappa^t\le1$ gives the simplified bound in the main text.

\subsection{One-step expansion}
\begin{lemma}[Second-order expansion of one QR--Oja step]
\label{lem:oja-expansion}
Let $P=QQ^\top$, $Q^\top Q=I_K$, let $S=S^\top$, and let $P^+$ project onto $\operatorname{col}((I+\eta_qS)Q)$. If $\eta_q\|S\|_2\le1/8$, then
\begin{align}
 P^+={}&P+\eta_q[(I-P)SP+PS(I-P)]+R,\nonumber\\
 &\|R\|_2\le16\eta_q^2\|S\|_2^2.
 \label{eq:app-oja-expansion}
\end{align}
\end{lemma}

\begin{proof}
Set
\begin{equation}
 E=\eta_qS,\qquad e=\|E\|_2\le\tfrac18,\qquad
 \widetilde Q=(I+E)Q.
\end{equation}
The update changes the Gram matrix of the columns of $Q$.  Define its first- and second-order parts by
\begin{equation}
 A_E=Q^\top EQ,\qquad B_E=Q^\top E^2Q,
 \qquad \Delta_E=2A_E+B_E.
\end{equation}
Because $Q^\top Q=I_K$ and $E=E^\top$,
\begin{equation}
 \widetilde Q^\top\widetilde Q=I+\Delta_E.
\end{equation}
Moreover, $\|A_E\|_2\le e$, $\|B_E\|_2\le e^2$, and hence
\begin{equation}
 \|\Delta_E\|_2\le2e+e^2\le\frac{17}{64}<1.
\end{equation}
The exact resolvent identity $(I+\Delta_E)^{-1}=I-\Delta_E+\Delta_E^2(I+\Delta_E)^{-1}$ therefore gives

\begin{equation}
 (I+\Delta_E)^{-1}=I-2A_E+F_E,\qquad
 F_E=-B_E+\Delta_E^2(I+\Delta_E)^{-1}.
\end{equation}
Since $\|(I+\Delta_E)^{-1}\|_2\le(1-\|\Delta_E\|_2)^{-1}$,
\begin{equation}
 \|F_E\|_2
 \le e^2+\frac{(2e+e^2)^2}{1-2e-e^2}
 \le8e^2.
 \label{eq:app-inverse-remainder}
\end{equation}

QR does not change the column space of $\widetilde Q$, so
\begin{equation}
 P^+=\widetilde Q(\widetilde Q^\top\widetilde Q)^{-1}
 \widetilde Q^\top.
\end{equation}
Substitution of the inverse expansion, together with $QA_EQ^\top=PEP$, yields
\begin{equation}
 P^+=P+EP+PE-2PEP+R_E
 =P+(I-P)EP+PE(I-P)+R_E,
\end{equation}
where the terms of order at least two are collected exactly as
\begin{align}
 R_E={}&EPE-2EPEP-2PEPE-2EPEPE\nonumber\\
 &+(I+E)QF_EQ^\top(I+E).
 \label{eq:app-projector-remainder}
\end{align}
Using $\|P\|_2=\|Q\|_2=1$ and Equation~\ref{eq:app-inverse-remainder},
\begin{align}
 \|R_E\|_2
 &\le[1+2+2+2e+8(1+e)^2]e^2\nonumber\\
 &\le16e^2.
\end{align}
Finally, substituting $E=\eta_qS$ and renaming $R_E$ as $R$ proves Equation~\ref{eq:app-oja-expansion}.
\end{proof}

\subsection{Global inequality and local contraction}
Fix $t$ and abbreviate $P=P_{t-1}$. Put $u_t=K-\operatorname{tr}(P_{\star,t}P)$. In the orthonormal eigenbasis of $M_t$, write
\[
 P=\begin{bmatrix}X&Y\\Y^\top&Z\end{bmatrix},
 \qquad M_t=\begin{bmatrix}M_t^\parallel&0\\0&M_t^\perp\end{bmatrix}.
\]
Since $P^2=P$, $YY^\top=X-X^2$. Thus
\begin{align*}
 A_t&:=\operatorname{tr}(P_{\star,t}(I-P)M_tP)\\
 &=\operatorname{tr}(M_t^\parallel YY^\top)
       -\operatorname{tr}(M_t^\perp Y^\top Y)
 \ge\gamma\|Y\|_F^2 .
\end{align*}
If $\theta_i$ are the principal angles, then $u_t=\sum_i\sin^2\theta_i$ and $\|Y\|_F^2=\sum_i\sin^2\theta_i\cos^2\theta_i \ge u_t(1-u_t)_+$. Define the mean-zero increment
\[
 \zeta_t=-2\eta_q\operatorname{tr}
       (P_{\star,t}(I-P)(S_t-M_t)P).
\]
The expansion lemma and $|\operatorname{tr}(P_{\star,t}R)|\le K\|R\|_2$ give globally
\[
 \Phi_t\le u_t-2\eta_q\gamma u_t(1-u_t)_+
                 +16K\eta_q^2\|S_t\|_2^2+\zeta_t.
\]
The projector distance satisfies the triangle inequality, so $\sqrt{u_t}\le\sqrt{\Phi_{t-1}}+\omega_t$. When $u_t\le1/2$, the preceding inequality and
\[
 (1-2\eta_q a)(x+y)^2
 \le(1-\eta_q a)x^2+\frac{y^2}{\eta_q a}\quad(x,y\ge0)
\]
therefore imply
\begin{equation}
 \Phi_t\le\kappa\Phi_{t-1}
       +\frac{\bar\omega^2}{\eta_q a}
       +16K\eta_q^2\|S_t\|_2^2+\zeta_t .
 \label{eq:app-local-recursion}
\end{equation}
The following induction verifies $u_t\le1/2$ throughout the finite trajectory on a high-probability event.

\subsection{Concentration and invariant-region induction}
The matrix $P P_{\star,t}(I-P)$ is predictable, has rank at most $K$ and spectral norm at most one. Consequently
\[
 \E[\zeta_t\mid\mathcal F_{t-1}]=0,\quad
 |\zeta_t|\le4K\eta_q\Lambda,\quad
 \E[\zeta_t^2\mid\mathcal F_{t-1}]\le4K^2\eta_q^2\Sigma^2 .
\]
The variance bound follows by centering the scalar trace, whose second moment is at most $K^2\E[\|S_t\|_2^2\mid\mathcal F_{t-1}]$. For each fixed endpoint $t$, apply the scalar Freedman inequality \citep[Theorem~1.1]{tropp2011freedman} to $\sum_{s=1}^t\kappa^{t-s}\zeta_s$. For each fixed $t$, the weights $\kappa^{t-s}$ are deterministic; the partial sums in $s$ form a martingale. Its predictable variance is at most $4K^2\eta_q\Sigma^2/a$ and its increments are bounded by $4K\eta_q\Lambda$. The bound $\sqrt{2vx}+2bx/3$ for variance $v$ and increment bound $b$, followed by a union bound over $t\le T$, gives an event $\mathcal E$ such that
\[
 \Pr(\mathcal E\mid\mathcal E_0)\ge1-\delta,\qquad
 \sum_{s=1}^t\kappa^{t-s}\zeta_s\le b_T(\delta)
 \quad\text{for every }t\le T .
\]
Since $\mathcal E_0\in\mathcal F_0$, the same bound holds conditionally on $\mathcal E_0$. Write $v=\bar\omega^2/(\eta_q^2a^2)\le1/16$ and $w=16K\eta_q\Lambda^2/a$,
\[
 r_{\rm safe}^2
 \le \max\{\varepsilon_{\rm burn}^2,v\}+w+b_T(\delta)
 \le\tfrac1{16}+\tfrac18=\tfrac3{16}<\tfrac14.
\]
Thus the budgets imply $r_{\rm safe}\le1/2$, $\bar\omega\le1/64$, and $r_{\rm safe}+\bar\omega<1/\sqrt2$. On $\mathcal E_0\cap\mathcal E$, suppose inductively that $\Phi_j\le r_{\rm safe}^2$ for all $j<t$. Then $u_s\le(r_{\rm safe}+\bar\omega)^2<1/2$ for every $s\le t$. Iterating Equation~\ref{eq:app-local-recursion} and using the pathwise bound $\|S_s\|_2\le\Lambda$ yields
\[
 \Phi_t\le\kappa^t\Phi_0+(1-\kappa^t)B_{\max}+b_T(\delta)\le r_{\rm safe}^2.
\]
The base case is $\mathcal E_0$, completing the induction.

\subsection{Expected tracking error}
Condition on $\mathcal E_0\in\mathcal F_0$, so the noise increments retain their conditional mean-zero property. Write $\E_0,\Pr_0$ for expectation and probability under this conditioning, and define $J_t=\mathbf1\{\Phi_j\le r_{\rm safe}^2\text{ for every }0\le j\le t\}$ and $U_t=\Phi_tJ_t$. Because $J_{t-1}$ is predictable, $J_t\le J_{t-1}$, and $\Phi_t\ge0$, we have $U_t\le J_{t-1}\Phi_t$. On $\{J_{t-1}=1\}$ the local inequality applies; on its complement both $U_t$ and $U_{t-1}$ vanish. Predictability and mean-zero noise give
\[
 \E_0[U_t\mid\mathcal F_{t-1}]
 \le\kappa U_{t-1}
       +\frac{\bar\omega^2}{\eta_q a}
       +16K\eta_q^2\Sigma^2 .
\]
Iteration and $\Pr(J_t=0\mid\mathcal E_0)\le\delta$ imply
\[
 \E[\Phi_t\mid\mathcal E_0]\le
 \kappa^t\varepsilon_{\rm burn}^2+
 (1-\kappa^t)\left[\frac{\bar\omega^2}{\eta_q^2a^2}
                         +\frac{16K\eta_q\Sigma^2}{a}\right]+K\delta .
\]
Finally $\Phi_t\le K$ on burn-in failure. Substitute $a=\gamma/2$ and $\varepsilon_{{\rm trk},t}^2\le\Phi_t$ to obtain Equation~\ref{eq:theory-track-bound}.

\subsection{Proof of maintained-subspace capture}
\label{app:capture}
All projectors in this section are orthogonal. Fix a time $t$. Let $Q_t\in\mathbb R^{d\times K}$ be the maintained orthonormal basis, $V_t^\star\in\mathbb R^{d\times K}$ the top-$K$ eigenbasis of the population activation covariance $M_t$, and $U_t^\star\in\mathbb R^{d\times r}$ the shared rank-$r$ right gradient basis for the window $\mathcal I_t$. These three bases are held fixed for every gradient in the window. The preceding tracking theorem controls $\varepsilon_{{\rm trk},t}=\|P_{Q_t}-P_{V_t^\star}\|_2$; here we relate this error to gradient capture. The target $V_t^\star$ is defined by the conditional covariance $M_t$, not by a covariance averaged over future observations. For any nonzero matrix $G$, write
\[
 G(I-P_{Q_t})=G(I-P_{U_t^\star})(I-P_{V_t^\star})
 +GP_{U_t^\star}(I-P_{V_t^\star})
 +G(P_{V_t^\star}-P_{Q_t}).
\]
The first term has norm at most $\|G(I-P_{U_t^\star})\|_F$; the last has norm at most $\|G\|_F\varepsilon_{{\rm trk},t}$. Divide by $\|G\|_F$ and apply the triangle inequality in the product Hilbert space over $\mathcal I_t$, with weight $1/T_{\rm w}$. Assumptions~\ref{ass:theory-structure}--\ref{ass:theory-alignment} give
\[
 \left[\frac1{T_{\rm w}}\sum_{\tau\in\mathcal I_t}
 \frac{\|G_\tau(I-P_{Q_t})\|_F^2}{\|G_\tau\|_F^2}\right]^{1/2}
 \le\varepsilon_{{\rm sub},t}+\varepsilon_{{\rm act},t}
 +\varepsilon_{{\rm trk},t}.
\]
Orthogonality gives $\|G\|_F^2=\|GP_{Q_t}\|_F^2+\|G(I-P_{Q_t})\|_F^2$, proving Proposition~\ref{prop:theory-capture}. For an individual gradient, the same argument holds with its individual structural residuals. A window-average assumption alone does not imply the corresponding pointwise bound.

\newpage
\section{Energy Capture and Directional Alignment in Active Subspaces}
\label{app:usefulness}
\subsection{Fixed-subspace identity}
Fix $G\ne0$ and an ordered orthonormal maintained basis $Q\in\R^{d\times K}$ with $C_Q>0$. As in Section~\ref{sec:theory-active}, let $0<\theta<1$, $1\le k\le K$, and $h=\theta k\in\mathbb Z$. The active basis $Q_S$ contains the first $h$ columns of $Q$ and $k-h$ columns sampled uniformly without replacement from the remaining $K-h$ columns. Condition on the active basis $Q_S$. Let $B_S=GQ_S$ and $D=pk$. The fixed-subspace cosine identity of \citet[Theorem~5.4]{lin2026agzo}, applied to $Q_S$, gives
\[
 \E_R\cos(\langle G,RQ_S^\top\rangle_F RQ_S^\top,G)
 =\beta_D\frac{\|B_S\|_F}{\|G\|_F}.
\]
Indeed, the cosine is $|\langle B_S,R\rangle_F|/(\|G\|_F\|R\|_F)$. After aligning the first coefficient axis with $B_S$, the expected absolute first coordinate of a uniform unit vector is $\Gamma(D/2)/[\sqrt\pi\Gamma((D+1)/2)]$. We use the convention $\cos(0,G)=0$, including when $B_S=0$.

\subsection{Proof of the estimation-error decomposition}
Conditional on $Q_S$, write $B_S=GQ_S$ and $D=pk$. The Gaussian subspace moments of \citet[Theorem~2]{yu2025subzero} give $\E_R[\langle B_S,R\rangle_F R]=B_S$ and $\E_R[\langle B_S,R\rangle_F^2\|R\|_F^2]=(D+2)\|B_S\|_F^2$. Consequently, the conditional mean of $\widehat G_{k,0}$ is $GP_{Q_S}$, and its conditional variance is $(D+1)\|B_S\|_F^2$. The decomposition
\[
\widehat G_{k,0}-G
=(\widehat G_{k,0}-GP_{Q_S})-G(I-P_{Q_S})
\]
has orthogonal components in the Frobenius inner product. Hence
\[
\E_R\|\widehat G_{k,0}-G\|_F^2
=(pk+1)\|GP_{Q_S}\|_F^2+\|G(I-P_{Q_S})\|_F^2.
\]
This proves Equation~\ref{eq:theory-active-mse}.

\subsection{Proof of Proposition~\ref{prop:theory-active}}
Write $Q=[q_1,\ldots,q_K]$ and $e_j=\|Gq_j\|_2^2$. Each non-shared column has inclusion probability $\rho=(k-h)/(K-h)$. Since the selected columns are orthonormal,
\begin{align*}
\E_S\|GQ_S\|_F^2
&=\sum_{j=1}^h e_j+\rho\sum_{j=h+1}^K e_j\\
&=\|GP_Q\|_F^2[\tau_h+\rho(1-\tau_h)].
\end{align*}
Dividing by $\|G\|_F^2>0$ proves the capture identity in Equation~\ref{eq:theory-active-result}. Moreover, $A_{K,k}=1-(K-k)(1-\tau_h)/(K-h)$, so $\tau_h\ge1-\varepsilon$ implies $A_{K,k}\ge1-\varepsilon(K-k)/(K-h)$.

Under the Gaussian model, define
\[
\mathcal L_{K,k}:=\beta_{pk}\sqrt{C_Q}[\tau_h+\rho(1-\tau_h)].
\]
Let $X=\|GQ_S\|_F$ and $b=\|GP_Q\|_F>0$. Since $0\le X\le b$, $X\ge X^2/b$. The fixed-subspace identity therefore gives
\[
\E_{S,R}\cos(\widehat G_{k,0},G)
=\frac{\beta_{pk}}{\|G\|_F}\E_S X
\ge\beta_{pk}\sqrt{C_Q}[\tau_h+\rho(1-\tau_h)],
\]
which establishes $J_{K,k}\ge\mathcal L_{K,k}$. At $k=K$, $Q_S$ spans the maintained subspace, and the fixed-subspace identity gives $J_{K,K}=\beta_{pK}\sqrt{C_Q}$. Since $C_Q>0$, the sufficient condition in Proposition~\ref{prop:theory-active} makes the lower bound on $J_{K,k}$ strictly exceed this exact full-activation value.

\subsection{Supplementary directional-alignment bounds}
For comparison, concavity of the square root yields the upper bound
\[
\E_{S,R}\cos(\widehat G_{k,0},G)
\le\beta_{pk}\sqrt{C_Q[\tau_h+\rho(1-\tau_h)]}.
\]

\subsection{Separating maintained and active widths}
Substituting $h=\theta k$ into the exact bound gives
\[
\mathcal L_{K,k}=\beta_{pk}\sqrt{C_Q}
\left[\tau_{\theta k}
+\frac{(1-\theta)k}{K-\theta k}(1-\tau_{\theta k})\right].
\]
The gamma-ratio expansion $\beta_{pk}=\sqrt{2/(\pi pk)}[1+O((pk)^{-1})]$ gives the decomposition
\begin{equation}
\mathcal L_{K,k}=\sqrt{\frac{2C_Q}{\pi p}}
\left[\frac{\tau_{\theta k}}{\sqrt{k}}
+\frac{(1-\theta)(1-\tau_{\theta k})\sqrt{k}}{K-\theta k}\right]
\bigl[1+O((pk)^{-1})\bigr].
\label{eq:theory-width-decomposition}
\end{equation}

For fixed $\theta$, the continuous extension of the inclusion probability has derivatives
\[
\left.\frac{\partial\rho}{\partial K}\right|_k
=-\frac{(1-\theta)k}{(K-\theta k)^2}<0,
\qquad
\left.\frac{\partial\rho}{\partial k}\right|_K
=\frac{(1-\theta)K}{(K-\theta k)^2}>0.
\]
The same monotonicities hold on admissible integer widths.

At $k=K$, every column participates, $\rho=1$, and the fixed-subspace identity gives $\E\cos(\widehat G_{K,0},G)=\beta_{pK}\sqrt{C_Q}$, recovering \citet[Theorem~5.4]{lin2026agzo}. For a fixed-prefix selection $k=h$, the identity instead gives $\beta_{pk}\sqrt{C_Q\tau_h}$.

\subsection{Additional comparison under shared-energy concentration}
Since the shared subspace is contained in every active subspace, which is contained in the maintained subspace, orthogonal projection gives
\[
\|GP_{Q_h}\|_F^2
\le\|GP_{Q_S}\|_F^2
\le\|GP_Q\|_F^2.
\]
Since $\|GQ_S\|_F=\|GP_{Q_S}\|_F$, taking square roots, dividing by $\|G\|_F$, and applying the fixed-subspace identity yields
\[
\beta_{pk}\sqrt{\tau_{\theta k}C_Q}
\le \E_{S,R}\cos(\widehat G_{k,0},G)
\le \beta_{pk}\sqrt{C_Q}.
\]

Now suppose $\tau_{\theta k_0}\ge1-\varepsilon$, with $0\le\varepsilon<1$. Applying the lower bound at $k_0$ and the upper bound at $k_1$ gives
\begin{align*}
\E\cos(\widehat G_{k_0,0},G)
&\ge\beta_{pk_0}\sqrt{(1-\varepsilon)C_Q},\\
\E\cos(\widehat G_{k_1,0},G)
&\le\beta_{pk_1}\sqrt{C_Q}.
\end{align*}
Since $C_Q>0$, if $\beta_{pk_1}/\beta_{pk_0}<\sqrt{1-\varepsilon}$, then $\E\cos(\widehat G_{k_0,0},G)>\E\cos(\widehat G_{k_1,0},G)$.

\newpage
\section{Experimental Setup and Evaluation Protocol}
\label{app:experimental-setup}

\paragraph{Models, tasks, and metrics.} The main comparisons use OPT-2.7B, OPT-13B, OPT-30B, Qwen3-0.6B-Base, and Qwen3-8B-Base. The 11 tasks are RTE, BoolQ, SST-2, WiC, WSC, COPA, CB, MultiRC, ReCoRD, SQuAD, and DROP; OPT-2.7B is evaluated on all 11. The main OPT tables report RTE, BoolQ, SST-2, WiC, WSC, and SQuAD; the Qwen tables replace WSC with MultiRC. Qwen3-4B is used in supplementary experiments. Baselines are MeZO, LoZO, AGZO, HiZOO, CurvZO, and ZO-Muon, with zero-shot scores as a non-training reference. We report accuracy for classification tasks and F1 for extractive question answering. Additional metrics and task results appear in Appendix~\ref{app:additional-finetuning}.

\paragraph{Data split.} Following the MeZO sampling protocol~\citep{malladi2023mezo}, training and development subsets are sampled without overlap from each task's official training split. For the six-task main comparisons, the subsets contain 1,000 training and 500 development examples, except WSC, which uses 454/100. The training minibatch size is 16. The additional OPT-2.7B CB and COPA comparisons use 150/100 and 300/100, respectively. Methods compared on a task use the same sampled splits. The OPT-2.7B and OPT-13B comparisons use five evenly spaced development checkpoints; Qwen3-0.6B-Base runs evaluate every 500 training steps. The OPT-13B MeZO, CurvZO, and LoZO SQuAD runs select checkpoints by generated-answer development F1. Other runs follow their recorded evaluation schedules. The selected checkpoint is evaluated on a separate official validation set. In the OPT-2.7B eleven-task comparison, BoolQ, ReCoRD, SQuAD, and DROP use fixed 1,000-example official-validation subsets; the other tasks use their full official validation sets.

\paragraph{Training budget.} Training budgets are matched by perturbed and unperturbed forward evaluations, with approximately 40,000 training forwards per run; validation and checkpoint I/O are excluded from this budget. AIM-ZO uses 15 perturbed evaluations and one unperturbed centre evaluation per update, giving 2,500 updates. MeZO uses two perturbed evaluations per update, giving 20,000 updates. Update counts for the other baselines account for their respective evaluations per update.

\paragraph{Training hardware.} OPT-30B and Qwen3-8B-Base experiments use 140\,GiB NVIDIA H200 GPUs. The other main experiments run on NVIDIA RTX 4090 GPUs with 24 or 48\,GiB of memory. Training ablations run on RTX 4090 GPUs.

\paragraph{Method settings.} AIM-ZO uses BF16 and updates all model parameters. Each trainable layer maintains $K=128$ directions and activates $k=64$ per perturbation: $h=48$ shared directions and 16 directions sampled independently from the maintained tail. The maintained subspace is updated at every training step with Oja step size $\eta_q=0.3$. Each perturbation uses the rank-one construction and layer-wise Frobenius normalization in Section~\ref{sec:queries-update}. The learning rate is constant and weight decay is zero. Over $T=2{,}500$ updates, the perturbation scale follows $\epsilon_t=\epsilon_0[\tfrac14+\tfrac38(1+\cos(\pi t/T))]$, from the initial values in Table~\ref{tab:aimzo-hyperparameters} to $\epsilon_0/4$. Table~\ref{tab:baseline-hyperparameters} summarises the baseline hyperparameter search ranges.

\begin{table}[t]
\centering
\caption{AIM-ZO learning rates and initial perturbation scales.}
\label{tab:aimzo-hyperparameters}
\small
\begin{tabular}{lcc}
\toprule
Model & Learning rate $\eta$ & Initial scale $\epsilon_0$ \\
\midrule
OPT-2.7B & $1.5\times10^{-3}$ & $8\times10^{-5}$ \\
OPT-13B & $1\times10^{-3}$ & $8\times10^{-5}$ \\
OPT-30B & $5\times10^{-4}$ & $6\times10^{-5}$ \\
Qwen3-0.6B-Base & $1\times10^{-3}$ & $8\times10^{-5}$ \\
Qwen3-8B-Base & $1\times10^{-3}$ & $8\times10^{-5}$ \\
\bottomrule
\end{tabular}
\end{table}

\begin{table}[t]
\centering
\caption{Baseline hyperparameter search ranges and calls per update for OPT-13B and Qwen3-0.6B, with OPT-2.7B MeZO also included.}
\label{tab:baseline-hyperparameters}
\small
\begin{tabular}{lccc}
\toprule
Method & Calls/update & Learning rate & Perturbation scale \\
\midrule
MeZO & 2 & $(1\text{--}5)\times10^{-7}$ & $10^{-3}$ \\
CurvZO & 2 & $(1\text{--}8)\times10^{-7}$ & $5\times10^{-4}\text{--}10^{-3}$ \\
AGZO & 3 & $5\times10^{-7}$ & $10^{-4}\text{--}10^{-3}$ \\
HiZOO & 3 & $5\times10^{-7}$ & $5\times10^{-4}\text{--}10^{-3}$ \\
LoZO & 2 & $10^{-7}\text{--}10^{-6}$ & $10^{-3}$ \\
ZO-Muon & 5 & $10^{-2}$ & $10^{-3}$ \\
\bottomrule
\end{tabular}
\end{table}

\paragraph{Reporting.} Tables report mean scores in percent and sample standard deviations when available. Six-task averages are unweighted and are reported only when all six task scores are available. Additional evaluation records are given in Appendix~\ref{app:additional-finetuning}.

\paragraph{Ablation protocol.} Unless stated otherwise, training ablations follow the corresponding main comparison's data split, batch size, forward-evaluation budget, checkpoint-selection rule, and final evaluation set. Changed components and seed counts are specified with each ablation.

\paragraph{Peak-memory measurement.} Figure~\ref{fig:drop-memory} reports peak allocated GPU memory from complete parameter-update steps, varying batch size and padded tensor length on Qwen3-0.6B-Base/DROP. The measurement uses BF16 weights, SDPA attention, and a 24\,GiB RTX 4090, with no gradient checkpointing or key-value cache. Each configuration runs ten updates in a separate process; OOM denotes an observed failure. Further implementation details are in Appendix~\ref{app:runtime-memory}.

\section{Supplementary Experiments for Section~\ref{sec:theory}}
\label{app:structural-diagnostics}
We provide additional measurements for the three analyses in Section~\ref{sec:theory}: gradient capture by the maintained subspace, the relation between maintained and active widths, and one-sided estimation.

\subsection{Gradient structure and maintained-subspace capture}
\label{app:moving-diagnostics}
We evaluate whether gradients concentrate in a shared subspace, whether activations identify that subspace, and how well Oja maintains it during training. The first two measurements use both fixed checkpoints and consecutive training checkpoints; the third compares saved Oja subspaces with current-batch references.

\paragraph{Fixed-checkpoint protocol.} We evaluate RTE checkpoints of Qwen3-4B at steps 500 and 2500, OPT-2.7B at step 2000, and OPT-30B at step 1000. Three disjoint sets contain 128 activation-calibration batches, 16 gradient-fit batches, and 16 gradient-test batches. They construct the top-$K$ activation subspace, fit the shared rank-$r$ right gradient subspace, and evaluate residuals and capture, respectively. Early, middle, and late refer to \texttt{q\_proj} layers 0, 18, and 35 in Qwen3-4B; 0, 16, and 31 in OPT-2.7B; and 0, 24, and 47 in OPT-30B.

The frozen-checkpoint rank scan measures the squared structural residual by layer at $r\in\{1,8,32,64,128\}$. The width scan fixes the fitted shared space at $r=8$ and varies $K\in\{8,16,32,64,128\}$, reporting the arithmetic mean of the three layer-wise squared alignment residuals. Residuals are normalized by each test gradient's squared Frobenius norm before averaging over batches; they are held-out counterparts of the quantities in Assumptions~\ref{ass:theory-structure}--\ref{ass:theory-alignment}. Specifically, let $\widehat U_r$ be the rank-$r$ basis fitted on the gradient-fit batches and $\widehat V_K$ the width-$K$ activation basis fitted on the calibration batches. For the held-out set $\mathcal B_{\rm test}$ of size $B_{\rm test}=16$, we report
\[
e_{\rm sub}^2(r)=\frac1{B_{\rm test}}\sum_{b\in\mathcal B_{\rm test}}
\frac{\|G_b(I-P_{\widehat U_r})\|_F^2}{\|G_b\|_F^2},\qquad
e_{\rm act}^2(r,K)=\frac1{B_{\rm test}}\sum_{b\in\mathcal B_{\rm test}}
\frac{\|G_bP_{\widehat U_r}(I-P_{\widehat V_K})\|_F^2}{\|G_b\|_F^2}.
\]
The hats distinguish fitted subspaces from the theoretical window optimizer $U_t^\star$ and population target $V_t^\star$. The frozen-checkpoint capture comparison evaluates the calibration-based activation target and the checkpoint Oja subspace at $K=128$.

\paragraph{Moving training trajectory.} Figure~\ref{fig:structure-diagnostics} uses 16 consecutive pre-update checkpoints: steps 500--515 for Qwen3-4B and 2000--2015 for OPT-2.7B, on RTE (seed 42). At each distinct parameter value $W_t$, we compute
\[
 \bar G_t=\nabla_W\frac1{1000}\sum_{x\in\mathcal D_{\rm train}}\ell(x;W_t).
\]
BF16 checkpoint weights are promoted to FP32. Autograd gradients are accumulated over four-example microbatches with sample-count weighting, retaining the three selected \texttt{q\_proj} matrices. The first eight gradients fit a shared right space; the last eight evaluate its normalized capture. This temporal split tests transfer across training steps without fitting on the evaluation gradients. Unlike the minibatch gradients in Assumption~\ref{ass:theory-structure}, $\bar G_t$ is the exact mean over the finite training set.

\begin{table}[!htb]
\centering
\caption{Temporal held-out shared-space capture. The rank-$r$ space is fitted on eight steps and evaluated on the next eight.}
\label{tab:moving-structure}
\begin{tabular}{llrrrr}
\toprule
Model & Layer & $r=8$ & $r=32$ & $r=64$ & $r=128$\\
\midrule
Qwen3-4B & Early & .622 & .820 & .879 & .923\\
& Middle & .911 & .964 & .980 & .990\\
& Late & .992 & .994 & .995 & .995\\
OPT-2.7B & Early & .443 & .678 & .784 & .868\\
& Middle & .954 & .980 & .987 & .992\\
& Late & .923 & .966 & .981 & .991\\
\bottomrule
\end{tabular}
\end{table}

Middle and late layers exhibit concentrated shared mean-gradient structure, whereas early layers require a larger rank.

For activation alignment, a disjoint 500-example train-dev split provides activation observations at the first eight checkpoints. The fitted activation target and the rank-8 shared gradient space are evaluated using the last eight full-training-set gradients, with the same normalized residual formula as above. The fitted target provides an empirical window reference for $V_t^\star$.

\begin{table}[!htb]
\centering
\caption{Temporal held-out activation alignment residual $e_{\rm act}^2$ at fixed shared rank $r=8$.}
\label{tab:moving-alignment}
\begin{tabular}{llrrrrr}
\toprule
Model & Layer & $K=8$ & $K=16$ & $K=32$ & $K=64$ & $K=128$\\
\midrule
Qwen3-4B & Early & .2710 & .1242 & .0523 & .0491 & .0462\\
& Middle & .1120 & .0909 & .0636 & .0469 & .0370\\
& Late & .2101 & .1572 & .0636 & .0344 & .0225\\
OPT-2.7B & Early & .2439 & .1550 & .0526 & .0295 & .0209\\
& Middle & .0394 & .0223 & .0142 & .0064 & .0046\\
& Late & .2557 & .2252 & .1571 & .1398 & .1184\\
\bottomrule
\end{tabular}
\end{table}

The alignment residual decreases with $K$, although its magnitude differs substantially by layer; the late OPT layer retains a residual of .1184 at $K=128$.

\FloatBarrier
\paragraph{Saved Oja subspaces and current-batch references.}
\label{app:oja-current-comparison}
The saved Oja subspace is evaluated against $\bar G_t$ at every checkpoint. An exact-SVD top-128 space from the current training batch provides a separate reference (Table~\ref{tab:oja-full-gradient}).

\begin{table}[!htb]
\centering
\caption{Capture of the finite-training-set mean gradient over 16 steps at $K=128$; the current-batch reference uses exact SVD.}
\label{tab:oja-full-gradient}
\begin{tabular}{llrrr}
\toprule
Model & Layer & Oja mean & Oja range & Current SVD mean\\
\midrule
Qwen3-4B & Early & .522 & .476--.628 & .695\\
& Middle & .862 & .791--.951 & .890\\
& Late & .878 & .699--.957 & .926\\
OPT-2.7B & Early & .635 & .587--.675 & .717\\
& Middle & .951 & .932--.966 & .964\\
& Late & .657 & .537--.719 & .770\\
\bottomrule
\end{tabular}
\end{table}

The exact current-batch SVD has higher capture in all six layers.

A second control reconstructs current-batch spaces using the actual training batch and five-step subspace iteration with Rayleigh--Ritz extraction (PI-5). At each checkpoint, both banks are evaluated against FP32 autograd gradients of 16 independent 16-example batches from the disjoint 500-example train-dev split. Four-example microbatch accumulation preserves each batch mean. Each layer has 256 paired batch/checkpoint comparisons. Active width 64 uses the average over 15 selections of the first 48 columns plus 16 random tail columns. The saved Qwen trajectory uses $\eta_q=.3$.

\begin{table}[!htb]
\centering
\caption{Oja minus PI-5 current-batch capture in percentage points (pp), and Oja win rates. Pairs share checkpoints and are not independent training runs.}
\label{tab:oja-current-disjoint}
\begin{tabular}{llrrrr}
\toprule
& & \multicolumn{2}{c}{Wide subspace ($K=128$)}
& \multicolumn{2}{c}{Active space ($k=64$)}\\
Model & Layer & Difference (pp) & Win rate & Difference (pp) & Win rate\\
\midrule
Qwen3-4B & Early & -11.06 & 3.9\% & -10.31 & 11.7\%\\
& Middle & +4.71 & 98.8\% & +4.06 & 94.5\%\\
& Late & +1.02 & 49.6\% & +2.38 & 64.5\%\\
OPT-2.7B & Early & +8.95 & 96.1\% & +6.29 & 83.2\%\\
& Middle & +1.54 & 80.9\% & +2.19 & 83.6\%\\
& Late & +4.82 & 61.3\% & +4.68 & 63.3\%\\
\bottomrule
\end{tabular}
\end{table}

Oja improves mean independent-batch capture in five of the six measured layers, with the early Qwen layer as the exception.

Oja maintenance takes 0.124\,s versus 0.229\,s for current-batch PI basis construction, a 45.9\% reduction; total step times are 0.774\,s and 0.760\,s, respectively.

\FloatBarrier
\paragraph{Stationary Oja replay.}
At frozen checkpoints, Oja is replayed over the calibration activation stream from three random initializations. Table~\ref{tab:stationary-replay} reports final normalized projector error and held-out gradient capture. The preferred step size varies by layer: a larger step can reduce transient error, whereas smaller steps can reduce the observation floor.

\begin{table}[!htb]
\centering
\caption{Stationary replay: (normalized $\Phi$, gradient capture). Qwen and OPT use steps 500 and 2000, respectively.}
\label{tab:stationary-replay}
\begin{tabular}{llrrrr}
\toprule
Model & Layer & $\eta_q=.03$ & $.1$ & $.3$ & $.6$\\
\midrule
Qwen3-4B & Early & (.943,.140) & (.926,.315) & (.892,.414) & (.837,.549)\\
& Middle & (.527,.898) & (.330,.901) & (.338,.894) & (.410,.886)\\
& Late & (.533,.924) & (.618,.923) & (.657,.922) & (.670,.921)\\
OPT-2.7B & Early & (.161,.713) & (.263,.677) & (.349,.646) & (.384,.635)\\
& Middle & (.349,.943) & (.417,.940) & (.559,.935) & (.617,.933)\\
& Late & (.440,.704) & (.572,.662) & (.663,.648) & (.697,.644)\\
\bottomrule
\end{tabular}
\end{table}

\FloatBarrier
\subsection{Active-width diagnostics}
\label{app:active-width-diagnostics}
The following scan compares gradient capture and single-direction alignment across active widths.

\begin{table}[!htb]
\centering
\caption{Full active-space width scan using independent Gaussian coefficient matrices. Each cell gives learned capture / random capture / mean learned single-direction first-order cosine.}
\label{tab:offline-width}
\begin{tabular}{rrr}
\toprule
Width & Qwen3-4B & OPT-2.7B\\
\midrule
8 & .7022 / .0020 / .002272 & .8137 / .0026 / .002855\\
16 & .7758 / .0068 / .001615 & .8458 / .0043 / .002052\\
32 & .8243 / .0121 / .001072 & .8736 / .0082 / .001374\\
64 & .8624 / .0277 / .000829 & .8975 / .0183 / .000982\\
128 & .8866 / .0510 / .000564 & .9166 / .0416 / .000752\\
\bottomrule
\end{tabular}
\end{table}

Capture increases while single-direction estimate alignment decreases as coefficient dimension grows. At active width 64, first-48 plus random-tail-16 capture is .8526 on Qwen and .8928 on OPT, compared with .8624 and .8975 for top-64.

\FloatBarrier
\subsection{One-sided and two-sided estimation}
\label{app:offline-population}
We examine one-sided and two-sided estimates under a fixed objective, a finite-training-set gradient reference, and the BF16 rank-one implementation.

\paragraph{Fixed-objective Gaussian test.} The checkpoint, minibatch, top-64 basis, and three representative \texttt{q\_proj} matrices are fixed. OP uses $N$ positive endpoints with population centering and denominator $N-1$; TP uses $N/2$ pairs, for the same number of perturbed loss evaluations. Estimates are compared with the exact autograd gradient of that minibatch, including energy outside the active subspace. Across five perturbation scales and $N\in\{4,8,16\}$, FP32 OP has lower mean MSE in all 15 settings per model, in both learned and random spaces. At $N=16$ and $\epsilon=8\times10^{-5}$, learned-space MSE reductions are 45.8\% on Qwen and 48.1\% on OPT. Unnormalized Gaussian BF16 tests favor OP in only 3/30 learned-space and 4/30 random-space settings; they are numerically sensitive and differ from the normalized implementation.

\paragraph{Matched perturbations in learned, random, and full spaces.} A separate FP32 control uses full Gaussian $AQ^\top$ for learned and random spaces and full iid Gaussian matrices for full-space perturbations. Directions are layer-Frobenius normalized to $\sqrt{pd}$; structured estimates are divided by the induced covariance factor $d/K$. At batch size 1, $K=64$, $\epsilon=8\times10^{-5}$, and 16 repeated perturbation populations, we obtain Table~\ref{tab:matched-probes}. OP16 uses 16 one-sided directions and TP8 uses eight two-sided directions; both require 16 perturbed loss evaluations.

\begin{table}[!htb]
\centering
\caption{Matched normalized Gaussian control against a fixed batch gradient. MSE is normalized by gradient energy; cosine gap is OP minus TP.}
\label{tab:matched-probes}
\begin{tabular}{llrrrr}
\toprule
Model & Space & OP MSE & TP MSE & Reduction & Cosine gap\\
\midrule
Qwen3-4B & Learned & 50,357 & 94,002 & 46.4\% & $1.192\times10^{-3}$\\
& Random & 1,042 & 1,680 & 38.0\% & $1.911\times10^{-4}$\\
& Dense & 1,825,752 & 3,487,915 & 47.7\% & $1.647\times10^{-4}$\\
OPT-2.7B & Learned & 30,755 & 65,777 & 53.2\% & $1.129\times10^{-3}$\\
& Random & 764 & 1,369 & 44.2\% & $2.530\times10^{-4}$\\
& Dense & 1,115,083 & 1,998,325 & 44.2\% & $2.280\times10^{-4}$\\
\bottomrule
\end{tabular}
\end{table}

OP lowers MSE in all three spaces. The learned-minus-random interaction in the OP--TP cosine gap is $1.001\times10^{-3}$ on Qwen (95\% Monte Carlo interval $[7.64\times10^{-4},1.24\times10^{-3}]$) and $8.756\times10^{-4}$ on OPT ($[3.58\times10^{-4},1.39\times10^{-3}]$). Against full-space perturbations the interactions are $1.027\times10^{-3}$ ($[8.17\times10^{-4},1.24\times10^{-3}]$) and $9.006\times10^{-4}$ ($[4.50\times10^{-4},1.35\times10^{-3}]$).

\FloatBarrier
\paragraph{Finite-training-set gradient.} OP16 and TP8 estimates from 16 mutually disjoint 16-example minibatches are compared with the exact 1000-example mean gradient at a fixed checkpoint. The selected three matrices are evaluated jointly using FP32 full-Gaussian structured or full-space perturbations. This protocol incorporates minibatch noise.

\begin{table}[!htb]
\centering
\caption{Finite-training-set reference at $\epsilon=8\times10^{-5}$, $K=64$, and 16 loss evaluations.}
\label{tab:full-gradient-probes}
\begin{tabular}{llrrrr}
\toprule
Model & Space & OP MSE & TP MSE & Reduction & Cosine gap\\
\midrule
Qwen3-4B & Learned & 1,995,835 & 3,743,201 & 46.7\% & $+7.47\times10^{-5}$\\
& Random & 45,462 & 91,242 & 50.2\% & $+3.06\times10^{-5}$\\
& Dense & 96,707,409 & 192,548,225 & 49.8\% & $-3.72\times10^{-5}$\\
OPT-2.7B & Learned & 1,374,688 & 2,154,535 & 36.2\% & $+8.44\times10^{-5}$\\
& Random & 31,731 & 53,723 & 40.9\% & $+7.29\times10^{-5}$\\
& Dense & 63,415,059 & 116,824,528 & 45.7\% & $-2.92\times10^{-5}$\\
\bottomrule
\end{tabular}
\end{table}

\begin{table}[!htb]
\centering
\caption{Perturbation-scale control against the exact training gradient. Entries are OP MSE reduction (\%) / OP--TP cosine gap in units of $10^{-5}$.}
\label{tab:probe-epsilon}
\begin{tabular}{llrrr}
\toprule
Model & $\epsilon$ & Learned & Random & Dense\\
\midrule
Qwen3-4B & $6\times10^{-5}$ & 46.7 / +7.64 & 50.2 / +2.99 & 49.8 / -3.69\\
& $7.43\times10^{-5}$ & 46.7 / +7.48 & 50.2 / +3.04 & 49.9 / -3.72\\
& $10^{-3}$ & 39.1 / -1.17 & 50.3 / +3.07 & 50.6 / -3.70\\
OPT-2.7B & $2.57\times10^{-5}$ & 36.5 / +7.66 & 41.1 / +7.49 & 45.7 / -2.68\\
& $6\times10^{-5}$ & 36.3 / +8.51 & 41.2 / +7.28 & 45.7 / -2.98\\
& $10^{-3}$ & 33.9 / +2.55 & 40.9 / +7.21 & 45.7 / -2.82\\
\bottomrule
\end{tabular}
\end{table}

At small tested scales, learned-space mean cosine favors OP whereas dense mean cosine favors TP in both models. At $10^{-3}$, the learned-space gain disappears on Qwen and shrinks on OPT. The paired 95\% intervals for learned-space gains include zero with 16 repetitions; OP has lower MSE in every tested space and scale.

\FloatBarrier
\paragraph{Implementation-level direction alignment.} This control compares rank-one $ABQ^\top$, layer-Frobenius-normalized BF16 perturbations at batch size 16 with dense Gaussian perturbations. OP uses 15 positive endpoints plus one center; TP uses eight pairs, each costing 16 loss evaluations. Cosine is measured against the exact gradient of the same fixed minibatch at $\epsilon=8\times10^{-5}$, with 16 paired repetitions.

\begin{table}[!htb]
\centering
\caption{BF16 implementation-level cosine. All numeric entries are in units of $10^{-5}$; intervals are paired 95\% intervals conditional on the checkpoint and batch.}
\label{tab:implementation-probes}
\begin{tabular}{llrrrr}
\toprule
Model & Space & OP & TP & Difference & 95\% interval\\
\midrule
Qwen3-4B & Learned & 38.72 & 32.93 & 5.791 & [1.96,9.63]\\
& Dense & 2.774 & 2.443 & .3307 & [-.514,1.18]\\
OPT-2.7B & Learned & 40.48 & 32.60 & 7.882 & [4.17,11.6]\\
& Dense & 5.482 & 4.246 & 1.236 & [.536,1.94]\\
\bottomrule
\end{tabular}
\end{table}

Learned-space directions have higher cosine alignment, and their mean OP--TP gain is larger in both models.
\FloatBarrier

\newpage
\section{Active-Subspace Ablations}
\label{app:component-ablations}

\subsection{Shared versus independently resampled active subspaces}
\label{app:active-resampling}
We compare using one common active basis for the perturbations within an update with independently resampling the candidate columns for each perturbation. Both variants retain the shared directions and use the same population estimator. Table~\ref{tab:active-resampling} reports development-selected official validation scores on OPT-2.7B, using five matched seeds and 40,000 training forward evaluations. Independent resampling gives higher mean scores on all four tasks in this comparison.

\begin{table}[!htb]
\centering
\caption{Common active basis versus per-perturbation resampling on OPT-2.7B. Classification scores are accuracy; SQuAD uses F1 (\%).}
\label{tab:active-resampling}
\small
\begin{tabular}{lrrrr}
\toprule
Variant & RTE & BoolQ & SST-2 & SQuAD\\
\midrule
MeZO & $65.13\pm1.67$ & $66.10\pm2.26$ & $92.50\pm.54$ & $80.99\pm1.42$\\
Common active basis & $66.14\pm.78$ & $66.84\pm.63$ & $92.48\pm.78$ & $81.06\pm1.02$\\
\method{} & $67.51\pm3.18$ & $67.22\pm1.67$ & $92.87\pm.47$ & $81.19\pm.49$\\
\bottomrule
\end{tabular}
\end{table}

\subsection{Maintenance refresh frequency}
\label{app:maintenance-frequency}
On OPT-2.7B RTE, we compare a fixed activation subspace with Oja updates every ten steps and every step. Table~\ref{tab:maintenance-frequency} reports final development accuracy after 2,500 updates over five seeds; per-step updates have the highest mean.

\begin{table}[!htb]
\centering
\caption{OPT-2.7B RTE maintenance-frequency ablation (five seeds; final development accuracy, \%).}
\label{tab:maintenance-frequency}
\begin{tabular}{lc}
\toprule
Maintenance rule & Accuracy \\
\midrule
Fixed subspace & $68.40\pm2.04$ \\
Oja every ten steps & $68.92\pm3.38$ \\
Oja every step & $70.20\pm1.95$ \\
\bottomrule
\end{tabular}
\end{table}

\subsection{Maintained width and fixed-prefix selection}
\label{app:width-training}
Table~\ref{tab:width-training} compares width-64 and width-128 maintenance with a fixed prefix of 64 active directions. The \method{} column gives the main-result reference at $K=128$, $k=64$, and $h=48$. The fixed-prefix variants use three seeds and 40,000 training forward evaluations; the main-result reference uses five seeds. All scores are official RTE accuracy on 277 examples. The two fixed-prefix widths use separately initialized maintained spaces, so this is a configuration ablation rather than a test of additional columns acting on an identical prefix.

\begin{table}[!htb]
\centering
\caption{RTE accuracy (\%) with fixed-prefix activation and the \method{} main-result reference.}
\label{tab:width-training}
\small
\begin{tabular}{lrrr}
\toprule
Model & $K=64,\ k=h=64$ & $K=128,\ k=h=64$ & \method{}\\
\midrule
OPT-2.7B & $64.26\pm3.82$ & $66.55\pm1.10$ & $\mathbf{67.51}\pm3.18$\\
Qwen3-0.6B & $77.38\pm1.71$ & $77.38\pm.42$ & $\mathbf{77.62}\pm.92$\\
\bottomrule
\end{tabular}
\end{table}

\subsection{Maintained width with sampled tails}
\label{app:sampled-maintained-width}
We vary $K$ while fixing $k=64$ and $h=48$ on OPT-2.7B RTE. All three variants use BF16, three seeds, 40,000 training forward evaluations, and an Oja refresh every ten steps. The $K=128$ row shares the three-seed rank-one setting of Table~\ref{tab:coefficient-rank}. At $K=64$, all 16 tail columns are selected; at larger widths, they are sampled from the maintained tail. Table~\ref{tab:sampled-maintained-width} reports the resulting accuracy.

\begin{table}[!htb]
\centering
\caption{OPT-2.7B RTE maintained-width scan at $k=64$ and $h=48$ (three seeds; official-validation accuracy, \%).}
\label{tab:sampled-maintained-width}
\begin{tabular}{rc}
\toprule
Maintained width $K$ & Accuracy \\
\midrule
64 & $65.70\pm1.25$ \\
128 & $65.46\pm3.62$ \\
256 & $67.15\pm0.72$ \\
\bottomrule
\end{tabular}
\end{table}

\subsection{Shared width}
\label{app:shared-width-training}
At fixed $K=128$ and $k=64$, we vary the number $h$ of shared directions in the same OPT-2.7B RTE protocol. The $h=48$ row reuses the rank-one control above; $h=64$ is the fixed-prefix OPT-2.7B result from Table~\ref{tab:width-training}. Table~\ref{tab:shared-width-training} reports the three-seed RTE comparison.

\begin{table}[!htb]
\centering
\caption{OPT-2.7B RTE shared-width scan at $K=128$ and $k=64$ (three seeds; official-validation accuracy, \%).}
\label{tab:shared-width-training}
\begin{tabular}{rcc}
\toprule
Shared width $h$ & Sampled tail $k-h$ & Accuracy \\
\midrule
0 & 64 & $60.77\pm1.99$ \\
48 & 16 & $65.46\pm3.62$ \\
64 & 0 & $66.55\pm1.10$ \\
\bottomrule
\end{tabular}
\end{table}

\subsection{Population size}
\label{app:population-size-training}
We compare eight and 15 one-sided population members on OPT-2.7B RTE at $K=128$, $k=64$, and $h=48$. The $N=8$ runs use 4,444 updates with nine forward evaluations per update (39,996 in total); the $N=15$ control uses 2,500 updates with 16 evaluations per update (40,000 in total). Both use three seeds and Oja refresh every ten steps. Table~\ref{tab:population-size-training} reports development-selected official-validation accuracy. The larger population has the higher mean under this matched forward-evaluation budget.

\begin{table}[!htb]
\centering
\caption{OPT-2.7B RTE population-size ablation (three seeds; official-validation accuracy, \%).}
\label{tab:population-size-training}
\begin{tabular}{rcc}
\toprule
Population size $N$ & Training forwards & Accuracy \\
\midrule
8 & 39,996 & $61.01\pm3.77$ \\
15 & 40,000 & $65.46\pm3.62$ \\
\bottomrule
\end{tabular}
\end{table}

\subsection{Perturbation coefficient rank}
\label{app:coefficient-rank}
Table~\ref{tab:coefficient-rank} compares rank-one, rank-four, and full Gaussian coefficients in the active subspace on OPT-2.7B RTE. Each variant uses three seeds, BF16, and 40,000 training forward evaluations. Mean accuracies are similar; step times summarize the median non-evaluation time of each training run.

\begin{table}[!htb]
\centering
\caption{Coefficient-rank ablation on OPT-2.7B RTE over three seeds. Accuracy is official-validation accuracy (\%); time is mean per-run median non-evaluation step time; memory is peak CUDA allocated memory.}
\label{tab:coefficient-rank}
\begin{tabular}{lccc}
\toprule
Coefficient form & Accuracy (\%) & Step time (s) & Memory (GiB) \\
\midrule
Rank-one $ab^\top Q^\top$ & $65.46\pm3.62$ & 10.35 & $11.214\pm0.158$ \\
Rank-four $ABQ^\top$ & $64.98\pm2.01$ & 9.77 & $11.216\pm0.160$ \\
Full $RQ^\top$ & $65.58\pm2.21$ & 12.24 & $11.153\pm0.105$ \\
\bottomrule
\end{tabular}
\end{table}
\FloatBarrier

\newpage
\section{Space Quality and One-Sided Estimation}
\label{app:population-quality}
\subsection{Online estimator ablation}
\label{app:online-op-tp}
We first compare OP15-RLOO with raw TP8 in complete OPT-2.7B RTE training runs over five matched seeds. The two variants share the training data, 2,500 updates, Oja maintenance, and perturbation settings. OP15-RLOO uses 15 positive evaluations and one centre evaluation per update; TP8 uses eight positive--negative pairs and one centre evaluation, giving 40,000 and 42,500 loss evaluations in total. Table~\ref{tab:online-op-tp} reports development accuracy. OP15-RLOO has a 1.36-point higher mean at the best-accuracy checkpoint and a 1.84-point higher mean after the final update.

\begin{table}[!htb]
\centering
\caption{Online OPT-2.7B RTE estimator ablation over five matched seeds (development accuracy, \%).}
\label{tab:online-op-tp}
\begin{tabular}{lcc}
\toprule
Estimator & Best checkpoint & Final update \\
\midrule
OP15-RLOO & $71.40\pm3.02$ & $71.12\pm3.29$ \\
Raw TP8 & $70.04\pm4.38$ & $69.28\pm4.62$ \\
\bottomrule
\end{tabular}
\end{table}

\subsection{Space quality at fixed active width}
\label{app:space-quality}
We isolate space quality at fixed active width $k=64$ by rotating a learned orthonormal basis toward an orthogonal random complement:
\[
 Q(s)=\sqrt{s}\,Q_{\rm learned}+\sqrt{1-s}\,Q_\perp,
 \qquad Q_{\rm learned}^{\top}Q_\perp=0.
\]
Each basis has the same dimension. The seven rotation levels are $s\in\{0,.1,.25,.5,.75,.9,1\}$, and capture is measured at each level. All levels use matched full-Gaussian coefficient matrices, FP32 arithmetic, and no layer-wise Frobenius normalization. OP uses 16 population-centered one-sided directions and TP uses eight two-sided pairs, each with 16 perturbed loss evaluations. For $Z_i=A_iQ(s)^\top$, let $y_i^\pm=f(W\pm\epsilon Z_i)$ and $\bar y^+=N^{-1}\sum_{i=1}^N y_i^+$. The compared estimators are
\begin{equation}
\widehat G_{\rm OP}
=\frac{1}{\epsilon(N-1)}\sum_{i=1}^N(y_i^+-\bar y^+)Z_i,
\qquad
\widehat G_{\rm TP}
=\frac1M\sum_{i=1}^M\frac{y_i^+-y_i^-}{2\epsilon}Z_i,
\label{eq:app-quality-estimators}
\end{equation}
where $N=16$ and $M=8$. We measure $\Delta_{\cos}=\cos(\widehat G_{\rm OP},G)-\cos(\widehat G_{\rm TP},G)$.

\paragraph{Exact-linear control.} At eight saved checkpoints per model, the reference is the exact mean gradient over the 1,000-example RTE training set. In this control, the endpoint values are $y_i^\pm=f(W)\pm\epsilon\langle G,Z_i\rangle_F$. Each rotation level uses 64 paired perturbation populations with an exact linear loss oracle. Table~\ref{tab:quality-linear} averages populations within each checkpoint and then averages the checkpoint means. The alignment gain increases with capture at every checkpoint; Pearson correlations over checkpoint-level quality means are .970 for Qwen3-4B and .965 for OPT-2.7B.

\begin{table}[!htb]
\centering
\caption{Fixed-width space-quality control with an exact linear oracle. Alignment gains are reported in units of $10^{-3}$.}
\label{tab:quality-linear}
\begin{tabular}{rrrrr}
\toprule
& \multicolumn{2}{c}{Qwen3-4B} & \multicolumn{2}{c}{OPT-2.7B}\\
$s$ & Capture & $\Delta_{\cos}$ & Capture & $\Delta_{\cos}$\\
\midrule
0 & .0046 & .0795 & .0020 & .0668\\
.10 & .0857 & .365 & .0931 & .448\\
.25 & .2072 & .567 & .2304 & .705\\
.50 & .4094 & .798 & .4598 & .996\\
.75 & .6116 & .975 & .6894 & 1.220\\
.90 & .7329 & 1.067 & .8274 & 1.337\\
1 & .8136 & 1.124 & .9199 & 1.410\\
\bottomrule
\end{tabular}
\end{table}

\paragraph{Real-loss control.} The corresponding FP32 forward evaluations cover three checkpoints per model and four independent one-example batches per checkpoint, with two perturbation scales and 16 paired populations per quality level. The reference is the exact gradient of the fixed batch objective. Across all 48 checkpoint/batch/scale settings, the mean alignment gain increases with capture. All 336 checkpoint/batch/scale/quality mean-MSE comparisons favor OP (Table~\ref{tab:quality-real}). These tests complement the rank-one BF16 implementation diagnostic in Table~\ref{tab:implementation-probes}.

\begin{table}[!htb]
\centering
\caption{Real-loss space-quality sweep. Correlation relates measured capture to the OP--TP alignment gain; monotone counts refer to checkpoint/batch settings.}
\label{tab:quality-real}
\small
\begin{tabular}{llrrr}
\toprule
Model & $\epsilon$ & Correlation & Monotone & OP MSE wins\\
\midrule
Qwen3-4B & $6\times10^{-5}$ & .952 & 12/12 & 84/84\\
& $10^{-3}$ & .952 & 12/12 & 84/84\\
OPT-2.7B & $2.57\times10^{-5}$ & .927 & 12/12 & 84/84\\
& $10^{-3}$ & .926 & 12/12 & 84/84\\
\bottomrule
\end{tabular}
\end{table}
\FloatBarrier

\newpage
\section{Runtime and Peak Memory}
\label{app:runtime-memory}

\subsection{Peak-memory scaling on DROP}
We measure Qwen3-0.6B-Base on DROP using BF16, SDPA attention, \texttt{use\_cache=False}, and no gradient checkpointing on a 24\,GiB RTX 4090. Each of 48 unique method/batch/length configurations runs ten actual parameter updates in a separate process. Peak memory is PyTorch CUDA allocated memory, including model weights and the training update; OOM denotes an observed out-of-memory failure. FO-SGD uses full-parameter SGD without momentum or FP32 master weights. \method{} uses rank-one perturbations, $K=128$, $k=64$, $h=48$, 15 population members, and factor restoration caching.

Figure~\ref{fig:drop-memory} shows that \method{} uses .242--.254\,GiB more memory than MeZO across the measured configurations and less than AGZO and FO-SGD. Inputs are padded to the stated tensor length, so the length sweep measures tensor-shape scaling.

\subsection{Complete runs at a matched evaluation budget}
Table~\ref{tab:runtime-full} reports accumulated training-step time from completed 40,000-evaluation runs, excluding validation and checkpoint I/O. Each row averages three matched seeds. Tasks and execution conditions differ across rows; the ratios are within-row comparisons.

\begin{table}[!htb]
\centering
\caption{Training-step hours at 40,000 forward evaluations. Ratio is \method{} time divided by MeZO time.}
\label{tab:runtime-full}
\begin{tabular}{llrrr}
\toprule
Model & Task & \method{} & MeZO & Ratio\\
\midrule
Qwen3-0.6B & RTE & 4.339 & 1.871 & 2.319\\
OPT-2.7B & RTE & 7.363 & 3.508 & 2.099\\
OPT-13B & BoolQ & 26.603 & 21.919 & 1.214\\
\bottomrule
\end{tabular}
\end{table}

\subsection{Large-model short-run measurements}
Each model and method runs ten BF16 training updates in an independent process, with the first batch matched across methods. Table~\ref{tab:runtime-short} reports mean step time over steps 6--10 and peak allocated memory; time per evaluation divides mean step time by the number of forward evaluations.

\begin{table}[!htb]
\centering
\caption{Large-model short-run time and memory.}
\label{tab:runtime-short}
\small
\begin{tabular}{llrrrr}
\toprule
Model & Method & Evaluations/step & s/step & s/eval. & GiB\\
\midrule
Qwen3-8B & MeZO & 2 & .653 & .327 & 19.929\\
& \method{} & 16 & 7.396 & .462 & 19.937\\
& AGZO & 3 & 1.196 & .399 & 18.772\\
OPT-30B & MeZO & 2 & 1.810 & .905 & 58.552\\
& \method{} & 16 & 14.052 & .878 & 60.034\\
& AGZO & 3 & 2.727 & .909 & 58.087\\
\bottomrule
\end{tabular}
\end{table}
\FloatBarrier

\subsection{Restoration caching}
On Qwen3-0.6B-Base/SQuAD, we compare three \method{} caching variants over ten updates on an RTX 4090 with matched batches. All variants produce identical parameter updates. Factor caching reduces mean step time by 20.6\% for an additional .036\,GiB of allocated memory (Table~\ref{tab:cache-ablation}); full-noise caching yields little further speedup.

\begin{table}[!htb]
\centering
\caption{Cache ablation for \method{}. Time averages all ten steps; allocated and reserved memory are reported separately.}
\label{tab:cache-ablation}
\small
\begin{tabular}{lrrr}
\toprule
Cache & s/step & Allocated GiB & Reserved GiB\\
\midrule
None & 9.164 & 10.251 & 19.111\\
Factor restoration & 7.273 & 10.287 & 19.143\\
Factor restoration + full noise & 7.248 & 10.576 & 20.516\\
\bottomrule
\end{tabular}
\end{table}
\FloatBarrier

\newpage
\section{Additional Fine-Tuning Results}
\label{app:additional-finetuning}
This section provides additional task results and evaluation metrics. Experimental settings are described in Appendix~\ref{app:experimental-setup}.
\paragraph{Main-table evaluation details.} OPT-2.7B CurvZO RTE and OPT-13B \method{} WiC use four seeds, while OPT-2.7B HiZOO WiC uses two. OPT-13B LoZO RTE and BoolQ use seeds 42, 142, and 242. MeZO, CurvZO, and LoZO SQuAD results select checkpoints by generated-answer development F1 and evaluate on an independent 1,000-example official-validation subset. The MeZO SQuAD entry uses full-parameter generation CE. A separate reproduction of MeZO's five-token-prefix F1 setup reaches $75.29\pm2.61$ F1 over three seeds; the main table uses the full-parameter setting.
\subsection{Complete OPT-2.7B task coverage}
\begin{table}[!htb]
\caption{Five-seed official-validation results for OPT-2.7B under 40,000 matched training forwards. Higher is better. W/T/L counts paired-seed outcomes for \textsc{AIM-ZO} relative to MeZO. Scores are percentages.}
  \label{tab:main_results}
  \centering
  \small
  \begin{tabular}{llccc}
    \toprule
    Dataset & Metric & MeZO & \textsc{AIM-ZO} & W/T/L \\
    \midrule
    RTE & Accuracy & $65.13\pm1.67$ & $\mathbf{67.51\pm3.18}$ & 4/1/0 \\
    BoolQ & Accuracy & $66.10\pm2.26$ & $\mathbf{67.22\pm1.67}$ & 5/0/0 \\
    SST-2 & Accuracy & $92.50\pm0.54$ & $\mathbf{92.87\pm0.47}$ & 5/0/0 \\
    WiC & Accuracy & $58.53\pm0.41$ & $\mathbf{60.13\pm1.70}$ & 5/0/0 \\
    MultiRC & Accuracy & $61.96\pm2.04$ & $\mathbf{63.94\pm1.59}$ & 4/0/1 \\
    MultiRC & F1a & $34.71\pm10.81$ & $\mathbf{48.91\pm4.09}$ & 5/0/0 \\
    ReCoRD & F1 & $88.07\pm0.37$ & $\mathbf{88.55\pm0.20}$ & 4/0/1 \\
    SQuAD & F1 & $80.99\pm1.42$ & $\mathbf{81.19\pm0.49}$ & 3/0/2 \\
    DROP & F1 & $24.50\pm0.46$ & $\mathbf{25.78\pm0.45}$ & 5/0/0 \\
    WSC & Accuracy & $54.81\pm1.52$ & $\mathbf{56.73\pm4.35}$ & 4/0/1 \\
    COPA & Accuracy & $\mathbf{80.80\pm2.17}$ & $80.60\pm3.21$ & 2/0/3 \\
    CB & Accuracy & $\mathbf{68.93\pm2.40}$ & $67.86\pm3.57$ & 2/0/3 \\
    CB & Macro-F1 & $\mathbf{53.66\pm8.83}$ & $49.36\pm6.82$ & 1/0/4 \\
    \bottomrule
  \end{tabular}
\end{table}
\FloatBarrier

Using each task's primary metric, \method{} has the higher five-seed mean on 9 of 11 tasks, with 43 wins, 1 tie, and 11 losses across the 55 paired task--seed comparisons. The largest gain is on MultiRC F1a. COPA is approximately tied, while CB favors MeZO under both accuracy and macro-F1.

\paragraph{Additional SQuAD metrics.} Zero-shot EM is 11.00 on OPT-2.7B and 21.90 on OPT-13B. On OPT-13B, MeZO, CurvZO, LoZO, and \method{} obtain EM scores of $44.13\pm1.76$, $51.73\pm1.26$, $34.23\pm6.79$, and $71.07\pm0.29$, respectively. On OPT-30B, MeZO and \method{} obtain $70.10\pm1.36$ and $72.88\pm0.88$ EM over five seeds.

\subsection{Qwen3-0.6B results}
\label{app:qwen06b-results}
Table~\ref{tab:qwen06b-multimethod} reports the six-task comparison on Qwen3-0.6B-Base. The task set replaces WSC with MultiRC. Trained methods use three seeds and approximately 40,000 training forward evaluations; SQuAD is evaluated on a fixed 1,000-example official-validation subset.
\providecommand{\expscore}[2]{$#1_{\scriptstyle\pm#2}$}
\begin{table}[!htb]
\centering
\caption{Multi-method comparison on Qwen3-0.6B-Base (\%). SQuAD uses F1; other tasks use accuracy. Avg. is the unweighted mean over six tasks; subscripts report available sample standard deviations.}
\label{tab:qwen06b-multimethod}
\small
\setlength{\tabcolsep}{2pt}
\begin{tabular*}{\textwidth}{@{\extracolsep{\fill}}lccccccc@{}}
\toprule
Method & RTE & BoolQ & SST-2 & WiC & MultiRC & SQuAD & Avg. \\
\midrule
Zero-shot & 58.48 & 62.20 & 58.26 & 51.25 & 59.08 & 56.22 & 57.58 \\
MeZO & \expscore{77.62}{1.11} & \expscore{73.87}{0.57} & \expscore{88.42}{1.35} & \expscore{55.36}{2.44} & \expscore{72.78}{1.49} & \expscore{83.23}{0.76} & 75.21 \\
CurvZO & \expscore{77.14}{1.04} & \expscore{68.28}{0.80} & \expscore{82.87}{1.15} & \expscore{49.74}{1.37} & \expscore{71.72}{2.86} & \expscore{71.29}{3.46} & 70.17 \\
HiZOO & \expscore{77.86}{0.55} & \expscore{72.76}{0.84} & \expscore{88.80}{0.88} & \expscore{53.92}{1.96} & \expscore{76.11}{0.93} & \expscore{80.42}{0.89} & 74.98 \\
AGZO & \expscore{69.68}{4.39} & \expscore{68.40}{0.26} & \expscore{84.75}{3.04} & \expscore{52.77}{3.70} & \expscore{73.89}{0.38} & \expscore{73.17}{1.43} & 70.44 \\
ZO-Muon & \expscore{75.69}{0.21} & \expscore{72.73}{0.79} & \expscore{88.76}{0.70} & \expscore{57.79}{2.05} & \expscore{75.10}{0.63} & \expscore{81.10}{1.36} & 75.20 \\
LoZO & \expscore{73.53}{0.21} & \expscore{69.88}{1.49} & \expscore{85.67}{0.75} & \expscore{52.93}{0.90} & \expscore{72.59}{0.70} & \expscore{76.74}{1.17} & 71.89 \\
\textsc{AIM-ZO} & \expscore{77.62}{0.92} & \expscore{73.84}{0.56} & \expscore{88.94}{0.70} & \expscore{54.92}{2.11} & \expscore{76.83}{0.54} & \expscore{81.97}{0.57} & 75.69 \\
\bottomrule
\end{tabular*}
\end{table}
\FloatBarrier

For SQuAD, zero-shot EM is 39.00. MeZO, CurvZO, HiZOO, AGZO, ZO-Muon, LoZO, and \method{} obtain EM scores of $74.57\pm1.12$, $60.13\pm4.01$, $71.77\pm1.15$, $61.63\pm1.77$, $71.67\pm1.45$, $66.30\pm1.57$, and $72.87\pm0.59$, respectively.

\subsection{Additional Qwen3-8B metrics}
\label{app:qwen8-results}
Table~\ref{tab:qwen8-extra-metrics} gives Qwen3-8B metrics omitted from the six-task comparison in Table~\ref{tab:scaling_results}. The MeZO and \method{} results use the same three seeds and development-selected checkpoints as the main table. The AGZO main-table row uses three seeds, BF16, and 39,999 training forward evaluations, with checkpoints selected from five development milestones.

\begin{table}[!htb]
\centering
\caption{Additional Qwen3-8B-Base evaluation metrics (\%; three seeds for trained methods).}
\label{tab:qwen8-extra-metrics}
\small
\begin{tabular}{lccc}
\toprule
Metric & Zero-shot & MeZO & \textsc{AIM-ZO} \\
\midrule
MultiRC F1a & 54.70 & $83.95\pm0.25$ & $84.48\pm0.37$ \\
MultiRC question EM & 28.23 & $54.60\pm1.37$ & $56.31\pm0.44$ \\
SQuAD EM & 72.40 & $78.20\pm2.23$ & $84.63\pm0.42$ \\
\bottomrule
\end{tabular}
\end{table}
\FloatBarrier

\newpage
\section{Additional Related Work}
\label{app:additional-related}

Table~\ref{tab:method-properties} summarizes the design differences among representative ZO methods.
\begin{table}[!htb]
\centering
\caption{Design properties of representative ZO methods. Sparse coordinate spaces count as subspace perturbations; guidance and history refer to information used in perturbation construction.}
\label{tab:method-properties}
\scriptsize
\resizebox{\linewidth}{!}{%
\setlength{\tabcolsep}{18pt}
\renewcommand{\arraystretch}{1}
\begin{tabular}{lcccc}
\toprule
Method & \shortstack{Subspace\\perturbations}
& \shortstack{Information-\\guided}
& \shortstack{History-\\informed}
& \shortstack{Width\\decoupling} \\
\midrule
MeZO & $\times$ & $\times$ & $\times$ & $\times$ \\
HiZOO & $\times$ & $\checkmark$ & $\checkmark$ & $\times$ \\
CurvZO & $\checkmark$ & $\checkmark$ & $\checkmark$ & $\times$ \\
LoZO & $\checkmark$ & $\times$ & $\times$ & $\times$ \\
SubZero & $\checkmark$ & $\times$ & $\times$ & $\times$ \\
ZO-Muon & $\checkmark$ & $\times$ & $\times$ & $\times$ \\
AGZO & $\checkmark$ & $\checkmark$ & $\times$ & $\times$ \\
\midrule
\method{} & $\checkmark$ & $\checkmark$ & $\checkmark$ & $\checkmark$ \\
\bottomrule
\end{tabular}}
\end{table}
\FloatBarrier

\paragraph{Fixed activation-informed subspaces.} ZO-Act constructs a fixed low-rank basis from an initial activation snapshot and uses it to restrict perturbations throughout fine-tuning \citep{dong2026zoact}. This avoids repeated basis construction, but does not refresh the activation-derived basis as model parameters change.

\paragraph{Sparse parameter selection.} Sparse MeZO selects parameters according to their magnitudes and sensitivity to ZO estimation noise \citep{liu2025sparsemezo}. Transferable Static Sparsity instead reuses a predetermined parameter mask throughout training \citep{guo2025staticsparsity}. These approaches restrict perturbations in parameter coordinates, whereas a matrix subspace can represent combinations of coordinates through its basis vectors.

\paragraph{Adaptive spaces and temporal gradient structure.} ASEBO estimates an adaptive space from the covariance of full-dimensional evolution-strategy gradient estimates and allocates function evaluations between that space and its orthogonal complement \citep{choromanski2019asebo}. AIM-ZO maintains layer-wise bases from forward activations and samples smaller active subspaces within them. Prior studies have observed that gradient information can concentrate in low-dimensional subspaces that remain stable or evolve gradually over portions of training \citep{gurari2018tiny,jaiswal2025stabilization}. First-order optimisers exploit related temporal structure: GaLore periodically refreshes projection bases from gradient SVDs, whereas Online Subspace Descent and SubTrack++ update or track them from backpropagated gradients \citep{zhao2024galore,liang2024osd,rajabi2025subtrack}.

\paragraph{Streaming subspace estimation.} Streaming PCA methods such as block Oja provide tools for incrementally estimating a leading eigenspace \citep{allenzhu2017oja,huang2021streaming}. Classical guarantees typically assume a stationary covariance; robust streaming PCA also studies Oja updates under temporally drifting covariances \citep{bienstock2022robust}. AIM-ZO uses this maintained activation space to construct ZO perturbations.

\end{document}